\documentclass[conference]{IEEEtran}
\IEEEoverridecommandlockouts
\usepackage{cite}
\usepackage{amsmath,amssymb,amsfonts,amsthm}
\usepackage{graphicx}
\usepackage{textcomp}
\usepackage{xcolor}
\def\BibTeX{{\rm B\kern-.05em{\sc i\kern-.025em b}\kern-.08em
    T\kern-.1667em\lower.7ex\hbox{E}\kern-.125emX}}

\usepackage{url}
\usepackage[hidelinks]{hyperref}
\usepackage[utf8]{inputenc}
\usepackage[small]{caption}
\usepackage{subcaption}
\usepackage{enumitem}
\usepackage{booktabs}
\usepackage{algorithm}
\usepackage{algpseudocode}
\usepackage{multirow}
\newtheorem{definition}{Definition}[section]
\newtheorem{theorem}{Theorem}[section]

\newtheorem{proposition}[theorem]{Proposition}

\begin{document}

\title{Fractional SDE-Net: Generation of Time Series Data with Long-term Memory}
\author{\IEEEauthorblockN{1\textsuperscript{st} Kohei Hayashi}
\IEEEauthorblockA{\textit{Graduate School of Mathematical Science} \\
\textit{The University of Tokyo}\\
Tokyo, Japan \\
kohei@ms.u-tokyo.ac.jp}
\and
\IEEEauthorblockN{2\textsuperscript{nd} Kei Nakagawa}
\IEEEauthorblockA{\textit{Innovation Lab} \\
\textit{Nomura Asset Management Co.,}\\
Tokyo, Japan \\
kei.nak.0315@gmail.com}
}

\maketitle

\begin{abstract}
In this paper, we focus on the generation of time-series data using neural networks. 
It is often the case that input time-series data have only one realized (and usually irregularly sampled) path, which makes it difficult to extract time-series characteristics, and its noise structure is more complicated than i.i.d. type. 
Time series data, especially from hydrology, telecommunications, economics, and finance, exhibit long-term memory also called long-range dependency (LRD).
The main purpose of this paper is to artificially generate time series with the help of neural networks, making the LRD of paths into account.
We propose fSDE-Net: neural fractional Stochastic Differential Equation Network. 
It generalizes the neural stochastic differential equation model by using fractional Brownian motion with a Hurst index larger than half, which exhibits the LRD property. 
We derive the solver of fSDE-Net and theoretically analyze the existence and uniqueness of the solution to fSDE-Net. Our experiments with artificial and real time-series data demonstrate that the fSDE-Net model can replicate distributional properties well.
\end{abstract}


\section{Introduction}
Time series data appears in various areas and its modeling, which enables us to understand more about phenomena behind time evolution or to describe forthcoming scenarios is a fundamental issue~\cite{shumway2000time}. 
Recently, generative models specializing in time series data using deep neural networks (DNNs) are gathering attention. 
As to learning of population distribution, generative adversarial network (GAN) advocated in \cite{goodfellow2014generative} is a basic approach in spite of the great success of GANs with images \cite{gui2021review}. Their ability to generate and manipulate data across multiple domains has contributed to their success.
GAN is also used for generating time series data.
For instance, beginning from models using recurrent neural networks (RNN-GAN\cite{mogren2016c}, \cite{esteban2017real}), TimeGAN reflecting time series structure~\cite{yoon2019time}, QuantGAN focusing on financial time series such as stock price or exchange rate~\cite{wiese2020quant}, SigGAN using signature as a characteristic feature of time-series paths~(\cite{ni2020conditional},\cite{ni2021sig}).

In general, time series data is a series of data points indexed in time order and taken at successive equally spaced points in time. However, it is often the case where input time series data is irregularly sampled. 
Indeed, financial data is sampled only on weekdays though markets are ongoing also when stock exchanges are closed, for instance~\cite{nakagawa2020ric}. 
If sampling of input data is irregular, it is inadequate to learn the dynamics with time-homogeneous systems even when the models have a nutritious structure. As an approach that overcomes such a difficulty, neural ordinary differential equations~(ODE) which combine DNNs with numerical solver of differential equations broaden a new stream for generation of time-series data. 
\cite{chen2018neural} pointed out that ResNet~\cite{he2016deep} can be interpreted as a system of ODE by taking some continuation procedure. 
As an application, \cite{rubanova2019latent} used the neural ODE to generate irregular sample time series. If neural ODE is used instead of deep ResNet, then the number of parameters is much reduced, which saves a great deal of memory. Moreover, one can utilize the adjoint method~\cite{pontryagin1987mathematical} when computing gradient of a loss function, which enables an efficient calculation by solving another differential equations.

After proposal of the neural ODE concept, several papers introduced random effect, making use of stochastic differential equation (SDE) driven by the standard Brownian motion $B_t$ of the form
\begin{equation}
\label{eq:sde}
dX_t = b_{\theta_1} (t, X_t) dt + \sigma_{\theta_2} (t, X_t) dB_t , 
\end{equation}
with some initial condition. 
SDE has already seen widespread use for modeling real-world random phenomena~\cite{karatzas2012brownian,revuz2013continuous}.
Here drift and diffusion functions $b_{\theta_1}$ and $\sigma_{\theta_2}$ are given by some neural networks and $(\theta_1, \theta_2 )$ are parameters of these networks. 

For modeling randomness of observed data, standard Brownian motion is a natural choice if external noise is added in an independent and identically distributed~(i.i.d.) manner. 
However, some real time-series data has a more complex noise structure than the standard Brownian motion, for instance, when increments of noise process have a correlation in time. In addition, also when deriving a neural SDE from discrete DNNs by taking an infinite depth limit, the noise term is not necessarily limited to the standard Brownian motion. Indeed, even if the noise effect is imposed in an i.i.d. manner by dropout or random initialization, for instance, the random variables may have correlation across layers as learning progresses or some additional batch-originated randomness is implemented.

From this point of view, \textit{fractional Brownian motion}\cite{mandelbrot1968fractional} is a good candidate to model the noise effect in a more realistic way. 
Fractional Brownian motion (fBm) is a continuous stochastic process parameterized by a real number called Hurst index, say $H$, which takes a value between zero and one. It is known that fBm matches the standard Brownian motion when $H=1/2$ so that fBm is a generalization of the standard one. Moreover, as the Hurst index becomes large, fBm exhibits long-term memory also called long-range dependency (LRD) in time and has better regularity of sample paths~(see Fig \ref{fig:fbm_paths}.).

The fBm has played an increasingly important role in many fields of application such as hydrology, telecommunications, economics and finance~\cite{biagini2008stochastic,banna2019fractional}. 
Particularly in the realm of mathematical finance, several models using fBm are proposed and they show remarkable success as an approach to describe real markets appropriately~\cite{rostek2013note}. 
First, \cite{greene1977long} pointed out that there exists long-range dependence in common stock returns, while it is reported that volatility of markets is rough when observed in short term. 
Due to the aforementioned properties of fBm, these two perspectives are reproduced by fBm with Hurst index $H>1/2$ and $H<1/2$, respectively. On the other hand, the existence of arbitrage opportunity is often proved for models using fBm, albeit it is also shown that a cost for transaction makes the arbitrage trade impossible which seems to be a more natural description. These stylized facts about time series in real financial markets motivate us to use fBm as driving noise.

In this paper, we generalize the existing neural SDE to the fractional version, especially focusing on the case where $H>1/2$, to reflect the LRD of time series. 
To generate time series with such a property, we propose fSDE-Net: neural fractional Stochastic Differential Equation Network.
Our proposal, fSDE-Net, generalizes the noise modeled by a standard Brownian motion to fBm $B^H$:
\[
dX_t = b_{\theta_1} (t, X_t) dt + \sigma_{\theta_2} (t, X_t)  dB^H_t.
\]
It generalizes the neural SDE model by using fractional Brownian motion with Hurst index larger than half, which exhibits LRD property. 
The main purpose of this paper is to artificially generate time series with the help of neural networks, making the LRD property of paths into account.

Finally, our contributions are as follows:
\begin{itemize}
\item
We propose fSDE-Net by extending SDE-Net using fBm, based on theoretical results on the existence and uniqueness of the solution to fSDEs.

\item
We construct the numerical scheme for the generator of our method, neural differential equation driven by fBm, and prove convergence of discretized solutions by the classical Euler scheme.

\item
As a consequence of the theoretical analysis, we give a criterion to the choice of activation function of driving neural networks, to make the continuous generative model realizable. 

\item
As an application of the fSDE-Net, we construct a generative model for artificial and real time-series data. 
Our experiments demonstrate that calibrated generator of the model can replicate distributional properties of the original time series, especially LRD.
Our code is available at  \url{https://github.com/xxx}.
\end{itemize}

\subsection{Organization of this paper} 
First we summarizes related works in Section~\ref{sec_rw}.
Then background of this study, including the descriptions of fBm, R/S statistics and SDE driven by fBm is explained in Section~\ref{sec_Prelim}. 
After that, the fSDE-Net we propose in this paper and related discussions are described in Section~\ref{sec_fSDE}.
Next, the numerical scheme and implementation of fSDE-Net are explained in Section~\ref{sec_generative} and then numerical experiments on synthetic and real world datasets are conducted in Section~\ref{sec_exp}.
The numerical experiments are conducted in a rather simple setting to check first of all whether the implementation is possible. 
Finally we give the conclusion of this paper in Section~\ref{sec_con}. 

\section{Related Works}\label{sec_rw}
Broadly speaking, existing approaches to extend neural ODE can be categorized into two lines.
The first use SDEs based on Brownian motion as a way to insert noise into a system and the second instead consider  other process such as jump and controlled differential equations. Our fSDE-Net directly extend the first approach using fBm which generalizes Brownian motion.

\subsection*{Neural Stochastic Differential Equation Models}
There are two points of view to consider SDE to extend neural ODE. 
The first motivation is to take the randomness of input data itself into account. 
In \cite{tzen2019theoretical} and \cite{tzen2019neural}, SDE is introduced as a continuous extension of the deep latent Gaussian model whose discrete version is originally introduced in \cite{rezende2014stochastic}. 
Through the lens of stochastic control, \cite{tzen2019theoretical} provided a unified perspective on both sampling and variational inference in such generative models.
Meanwhile, \cite{kong2020sde} considered SDE to model epistemic uncertainty and proposed SDE-Net.
They theoretically analyze the existence and uniqueness of the solution to SDE-Net and particularly applied to out-of-distribution detection tasks. 
\cite{kidger2021neural} used the neural SDE as a generative model of time series. They demonstrated that the current classical approach to fitting SDEs can be viewed as a special case of (Wasserstein) GANs, bringing the neural and classical regimes together.

The other motivation to consider SDE comes from the randomness of neural network parameters. 
As to this perspective, \cite{liu2019neural} considered the stochastic term as regularization effect, dropout for instance, and they showed that generalization performance and adversarial robustness were improved by the noise effect. On the other hand, not restricting to model regularization effect, \cite{peluchetti2020infinitely} derived an SDE in a more general situation, assuming weight and bias are given randomly. 
They established the convergence of identity ResNets to SDE solutions as the number of layers increased to infinity.

\subsection*{Other Neural Differential Equation Models}
There are several other extensions of neural ODE. For instance, \cite{jia2019neural} and \cite{herrera2021neural} considers the case when there are jumps in observed data. Both models are described by a neural ODE with jumps at random times. 
\cite{jia2019neural} considers hybrid systems which evolve continuously over time, but may also be interrupted by stochastic events. They extended the framework of Neural ODEs with a stochastic process term that models temporal point processes with a piecewise-continuous latent trajectory. 
In contrast, \cite{herrera2021neural} introduced a mathematical framework to precisely describe the problem statement of online prediction and filtering of a stochastic process with temporal irregular observations.
They then used a neural ODE to model the conditional expectation between two observations, which jumps whenever a new observation is made.
On the other hand, \cite{kidger2020neural} demonstrated how controlled differential equations may extend the neural ODE model and propose the neural controlled differential equation model which is the continuous analogue of an RNN.

\section{Preliminaries}\label{sec_Prelim}
\subsection{Fractional Brownian Motion}
\begin{figure*}[t]
  \begin{minipage}[b]{0.19\linewidth}
    \centering
    \includegraphics[keepaspectratio, scale=0.2]{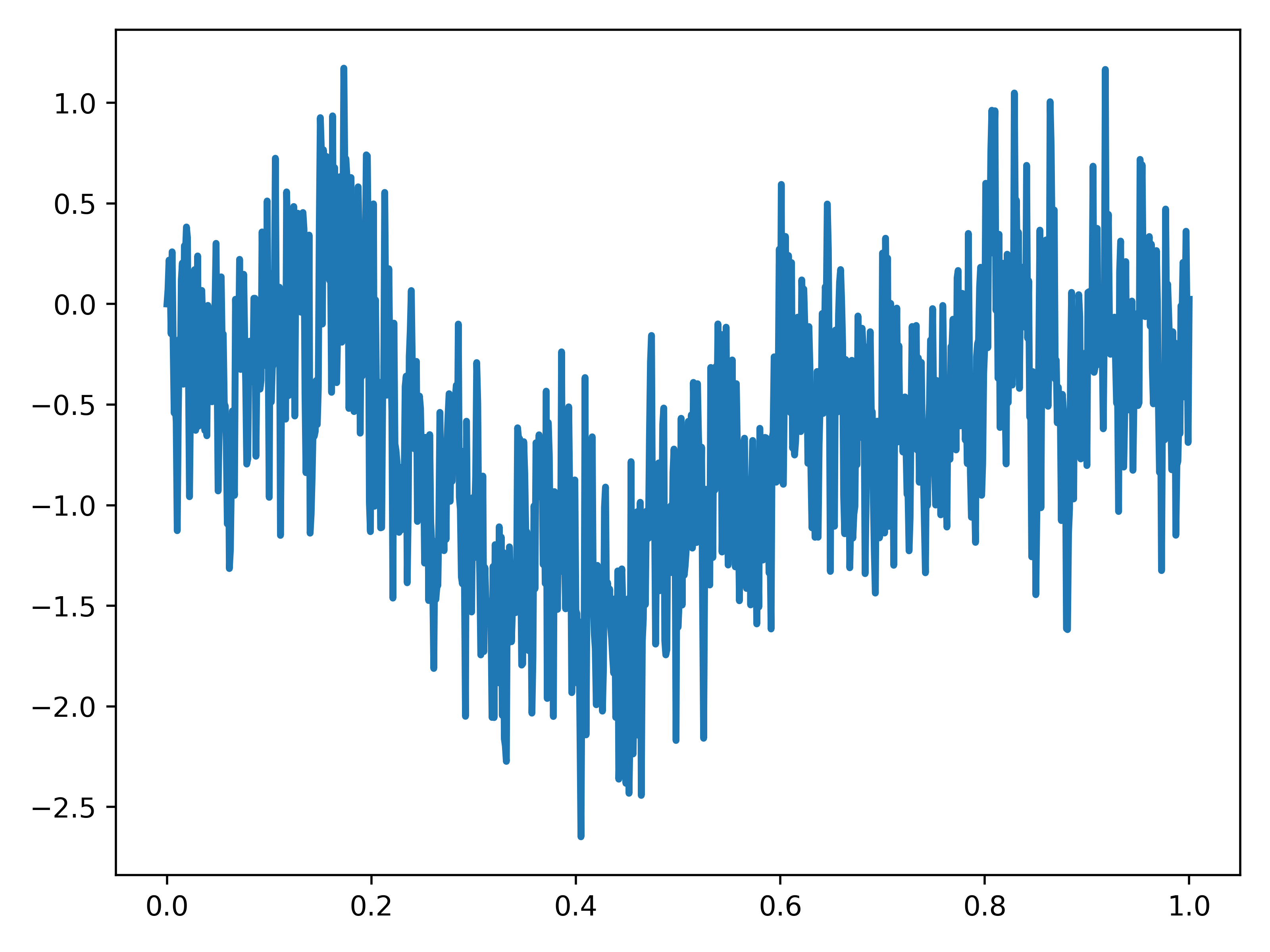}
    \subcaption{$H=0.1$}
  \end{minipage}
  \begin{minipage}[b]{0.19\linewidth}
    \centering
    \includegraphics[keepaspectratio, scale=0.2]{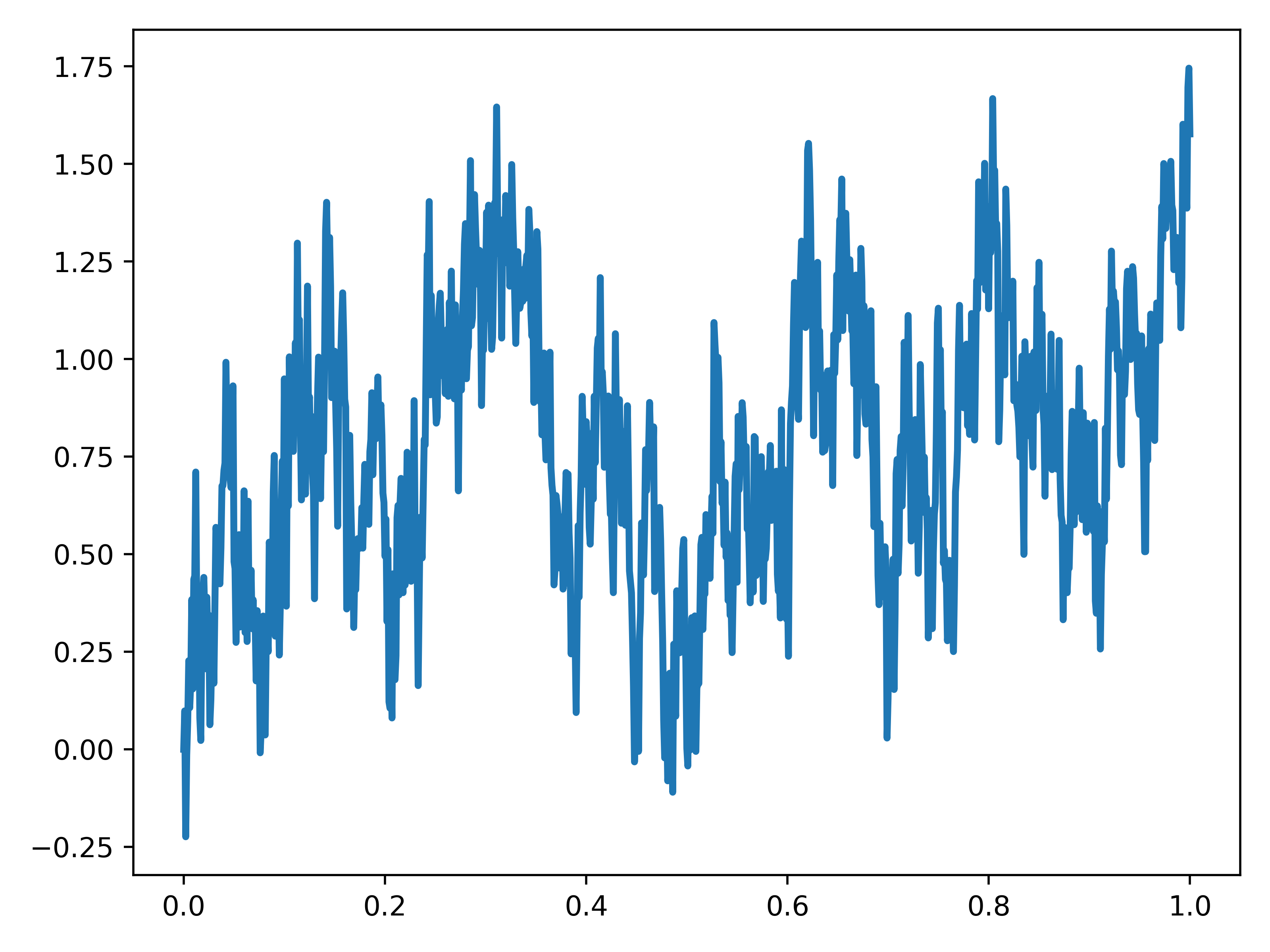}
    \subcaption{$H=0.25$}
  \end{minipage}
  \begin{minipage}[b]{0.19\linewidth}
    \centering
    \includegraphics[keepaspectratio, scale=0.2]{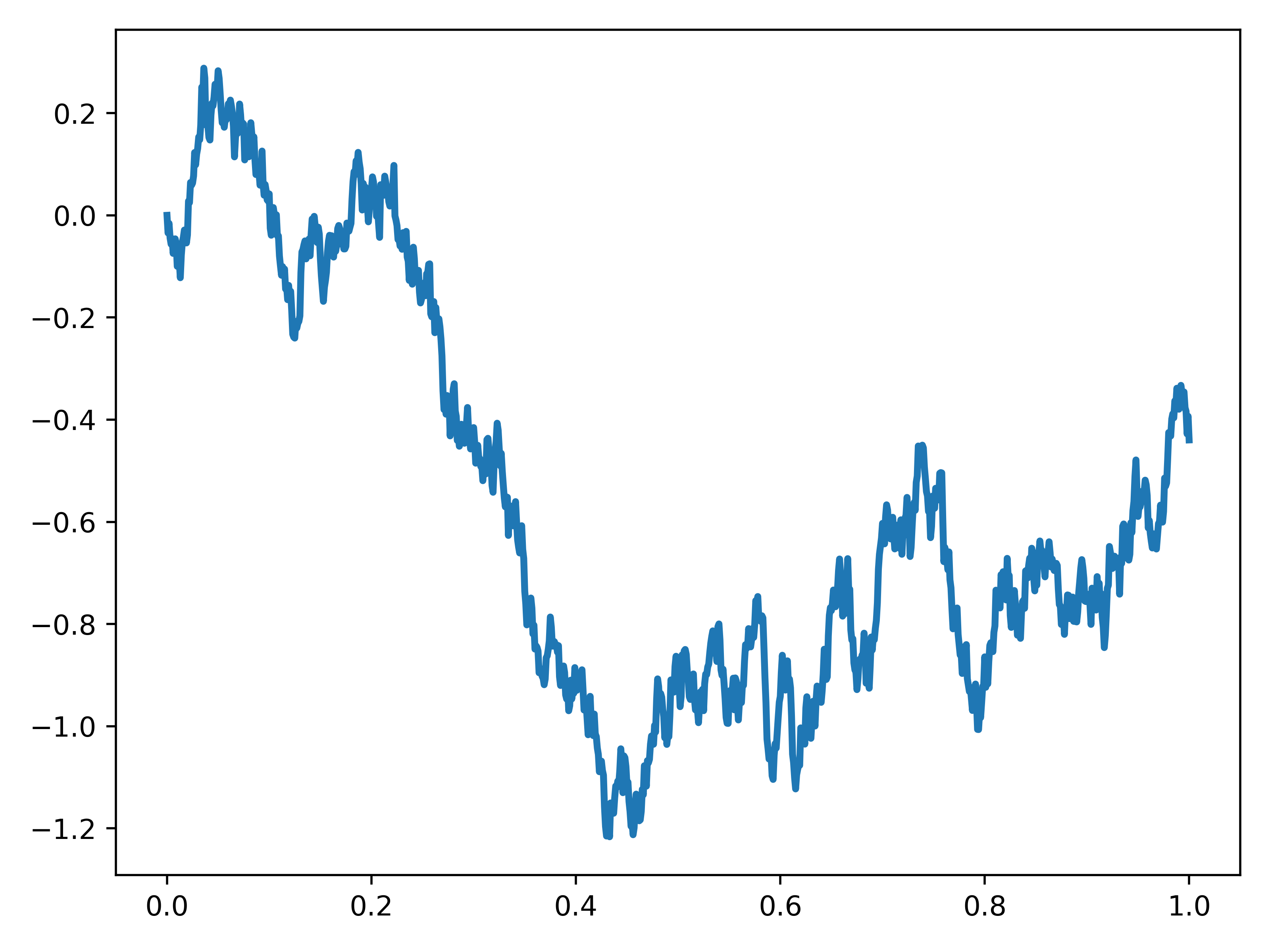}
    \subcaption{$H=0.5$}
  \end{minipage}
  \begin{minipage}[b]{0.19\linewidth}
    \centering
    \includegraphics[keepaspectratio, scale=0.2]{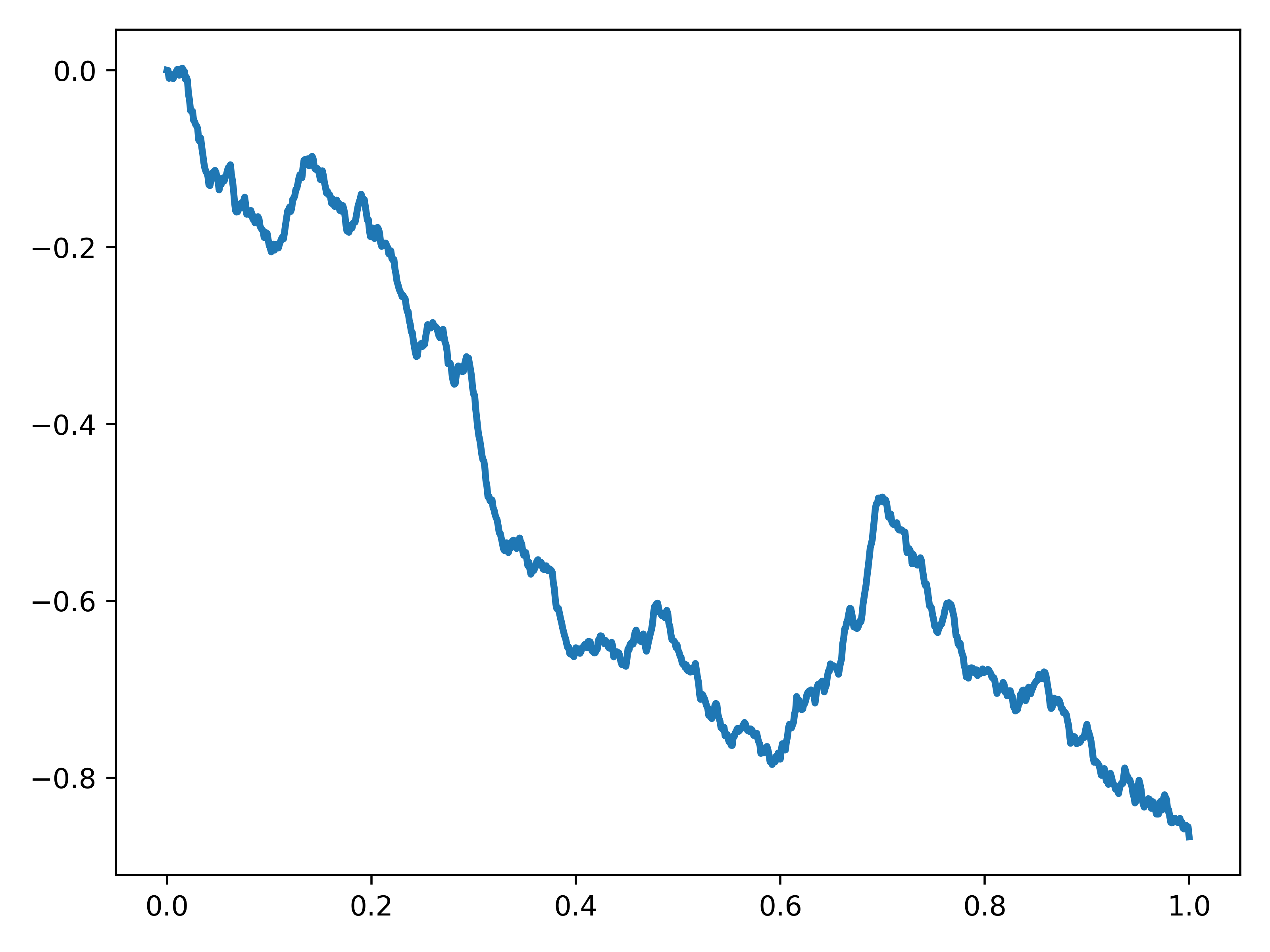}
    \subcaption{$H=0.75$}
  \end{minipage}
  \begin{minipage}[b]{0.19\linewidth}
    \centering
    \includegraphics[keepaspectratio, scale=0.2]{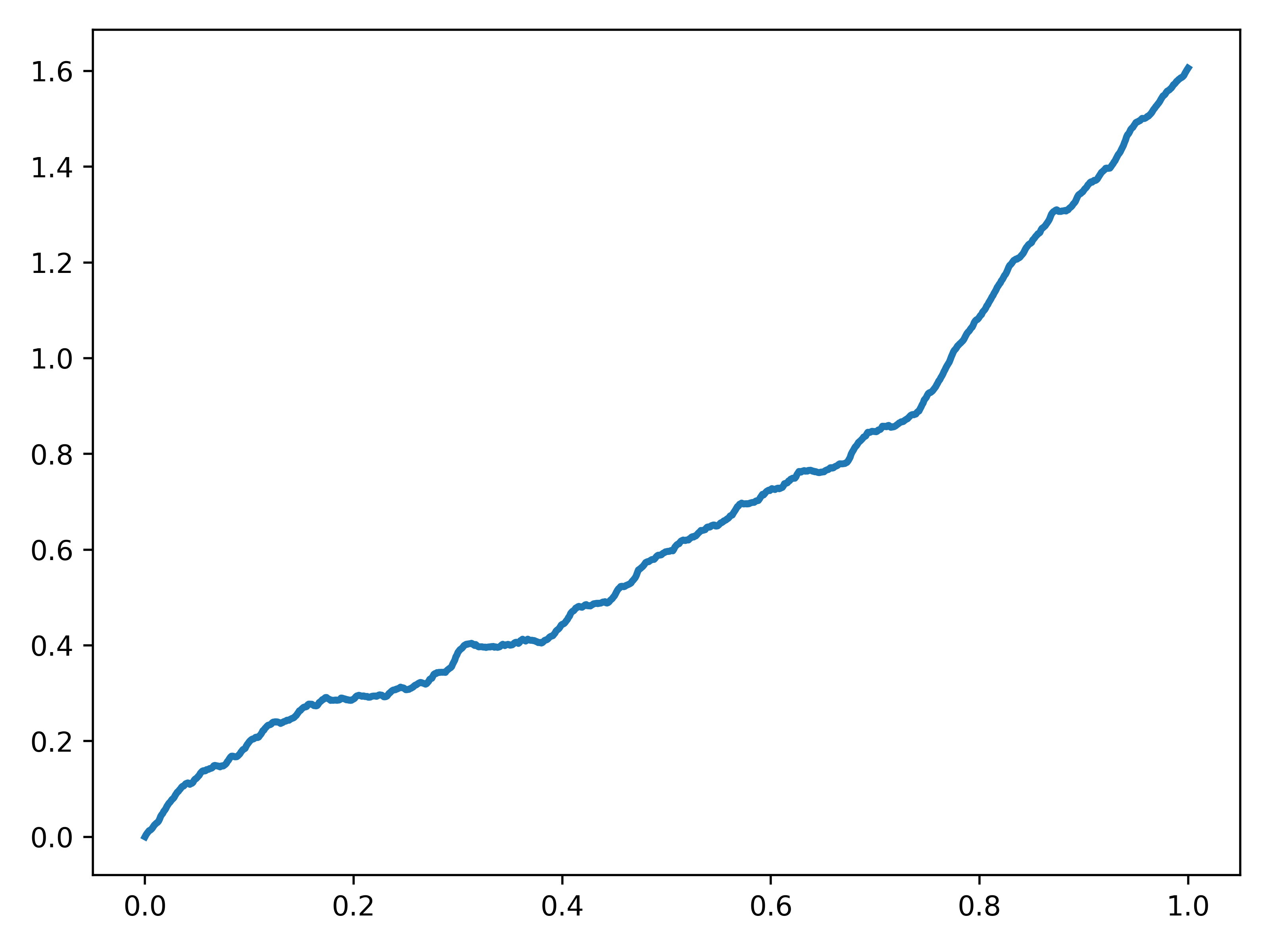}
    \subcaption{$H=0.9$}
  \end{minipage}
  \caption{Sample paths of fBm with various Hurst indices.}
  \label{fig:fbm_paths}
\end{figure*}

Here we review the definition and basic properties of fractional Brownian motion~(fBm). 
We first introduce a Gaussian process.
A Gaussian process $\{X_t\}_{t \ge 0 }$ is a collection of real-valued random variables $X_t$, all defined on the same probability space, such that any finite subset of ${X_t}$ has a Gaussian distribution.
Equivalently, $\{X_t\}_{t \ge 0 }$ is Gaussian process if every finite linear combination of ${X_t}$ is a Gaussian distribution. 
Then we define the fBm as follows.

\begin{definition}
Let $H \in (0,1) $ be a fixed number. We say that a real-valued mean-zero Gaussian process $B^H = \{ B^H_t \}_{t \ge 0 } $ is fractional Brownian motion~(fBm) with Hurst index $H$ if it satisfies $B^H_0 = 0 $ almost surely and \begin{equation}
\label{eq:fbm_def}
\mathrm{Cov}(B^H_s, B^H_t) 
= \frac{1}{2} (|t|^{2H} +|s|^{2H} - |t-s|^{2H}) 
\end{equation} 
holds for every $s, t \ge 0 $. 
\end{definition}
Note that fBm becomes the usual Brownian motion when $H=1/2$. 
From fBm we can define a discrete process $\{ B^H_{t+1} - B^H_t : t = 0, 1, \ldots \}$, which is called fractional Gaussian noise (fGn). It is straightforward from definition that increments of fBm are positively (reps. negatively) correlated if $H > 1/2$ (resp. $H<1/2$). Moreover, fBm has stationary increments since $B^H_t - B^H_s \sim \mathcal{N}(0,|t-s|^{2H})$. In particular, according to the Kolmogorov-Chentsov continuity criterion, almost all paths of fBm with Hurst index $H$ have $(H-\varepsilon)$-H\"{o}lder regularity for any $\varepsilon > 0$: there exists $C>0$ such that $
|B^H_t- B^H_s| \le C |t- s|^{H-\varepsilon} $ for every $s,t \ge 0$. In other words, sample paths of fBm have better regularity as the Hurst index becomes large. This can be observed in Fig \ref{fig:fbm_paths} which displays sample paths of fBm with various values of Hurst indices.

Next, we recall the notion of long-range dependency of stochastic processes. 
\begin{definition}
Let $X = \{ X_t \}_{t\ge 0}$ be a (generic) process and we write its increments by $X_{s,t} = X_t- X_s$. We say that the increments of the process $X$ exhibit long-range dependency if for all $h>0$, 
\[
\sum_{n=1}^\infty | \mathrm{Cov} (X_{0, h}, X_{(n-1)h, nh}) | = \infty . 
\]
\end{definition}

According to the relation \eqref{eq:fbm_def}, the increments of fBm exhibit long-range dependency if and only if $H>1/2$.

\subsection[R/S statistics]{$R/S$ statistics}
Following \cite{mandelbrot1972statistical}, we recall here the notion of $R/S$ statistic for testing whether long-range dependency exists in time-series data. Let $X = \{ X_t :t = 0, 1, \ldots, T \}$ be a discrete time series. We define the range of the process $X $ at time $T$ by 
\[
R_T = 
\max_{0 \le k \le T} \sum_{t=0}^k (X_t - \overline{X}_T)
- \min_{0 \le k \le T} \sum_{t=0}^k (X_t - \overline{X}_T) 
\]
where $\overline{X}_T = (1/T) \sum_{0 \le t \le T} X_t$ denotes the sample mean. Moreover, let $S_T$ be the sample standard deviation which is defined by
\[
S_T = \bigg( \frac{1}{T} \sum_{t=0}^T 
(X_t - \overline{X}_T )^2 \bigg)^{1/2} .
\]
Then the quantity $R_T/S_T$ is called the $R/S$ statistics. \cite{mandelbrot1972statistical} proved that the quantity $T^{-H} (R_T / S_T) $ for fGn converges in probability to a constant as $T$ goes to infinity. In particular, $R/S$ statistics asymptotically satisfy the relation 
\begin{equation}
\label{eq:RS}
\log(R_T/ S_T) =H \log T + \mathrm{const.} + o_p (1)
\end{equation}
where $o_p (1)$ is a quantity which vanishes in probability as $T$ tends to infinity. Thus Hurst index can be estimated by linear regression using the relation \eqref{eq:RS}. 

\subsection{SDE Driven by fBm}
To implement fSDE-Net, we study SDE where the standard Brownian motion is replaced by fBm. 
We refer the reader to \cite{karatzas2012brownian,revuz2013continuous} for a rigorous account of SDE.
Hereafter assume drift and volatility is homogeneous in time for simplicity and consider SDE driven by fBm of the form  
\begin{equation}
\label{eq:fSDE}
\begin{aligned}
 X_t = X_0 + \int_0^t b (X_s ) ds + \int_0^t \sigma (X_s) dB^H_s .
\end{aligned}
\end{equation}
To define solution to the equation \eqref{eq:fSDE}, one needs to define the integral with respect to fBm. 
For the classical case where $H=1/2$ so that the noise is actually the standard Brownian motion, we can define the third term in the RHS of \eqref{eq:fSDE} as the It\^{o} (or the Stratonovich) integral, using a martingale property. 
However, it is known that fBm with Hurst index $H$ is not even a semimartingale\footnote{Roughly speaking, a stochastic process is called a semimartingale if it can be decomposed as the sum of a (local) martingale and a right continuous with left limits adapted finite-variation process.} if $H \neq 1/2$. 
In particular, one cannot define integral with respect to fBm by means of the usual It\^{o} calculus in general. However, recall that the larger the Hurst index, sample paths of fractional Brownian motion has the better H\"{older} regularity. For the case of $H>1/2$, we can make use of this regularity improvement to define integral with respect to fBm as follows. In the sequel, we denote the set of all real-valued $\alpha$-H\"{o}lder continuous functions on $[0,T]$ by $C^\alpha ([0,T])$. Let $f \in C^\alpha([0,T])$ and $g \in C^\beta ([0,T])$ be real-valued deterministic functions. Then it is proved in \cite{young1936inequality} that if $\alpha + \beta > 1$, the integral 
\[
\int_0^T f (t) dg(t) = 
\lim_{|\Delta|\to 0} \sum_{i=1}^{n-1} f(t_i) (g(t_{i+1}) - g(t_i) )
\]
converges where $\Delta= \{ 0 = t_0 \le t_1 \le \cdots \le t_{n-1} \le t_n =T \}$ is a partition of $[0,T]$ and $|\Delta| = \max_{0 \le i \le n-1} |t_{i+1} - t_i|$. This integral is called Young integral. We say that a process $X = \{ X_t : t\in [0,T]\}$ is a solution of fractional SDE \eqref{eq:fSDE} if almost all paths of $X$ satisfy the equation \eqref{eq:fSDE} where the integral with respect to fBm is defined by Young integral. Now we are in a position to mention well-posedness of the fractional SDE \eqref{eq:fSDE}. It is expected that solutions of the equation \eqref{eq:fSDE} have the same regularity as fBm provided drift $b$ and volatility $\sigma$ are smooth enough, which leads to make integral with respect to fBm well-defined when $H>1/2$. Indeed, existence of a unique solution to the equation \eqref{eq:fSDE} are assured as follows. Let $C^m_b(\mathbb{R})$ be the set of all $C^m$-smooth real-valued functions whose derivatives up to order $m$ are bounded. 

\begin{proposition}
\label{prop:fSDEsolution}
Assume $H>1/2$ and $b, \sigma \in C^2_b (\mathbb{R})$. Then the equation \eqref{eq:fSDE} has a unique solution.
\end{proposition}

\section{fSDE-Net: Neural Fractional Stochastic Differential Equation Network}\label{sec_fSDE}
Here we describe our fSDE-Net and connect with the recent literature on neural ODEs and SDEs.
We begin with the origin of neural ODE when all variables are deterministic. 
First, recall that a building block of a ResNet with the residual mapping $f (h_t, \theta_t) $ can be described using the following form:
\begin{equation}
\label{eq:resnet}
h_{t+1} = h_t + f (h_t, \theta_t) 
\end{equation}
where $t \in \{ 0, \ldots, T\}$, $h_t$ is the input to the $i$-th residual network building block and $h_{t+1}$ is the corresponding output to the next building block. 
The parameters in this building block are denoted by $\theta_t$. 
In the relation \eqref{eq:resnet}, we take continuum limit by adding more layers and take small step size. 
Consequently, the resulting continuous dynamics of hidden units is described by an ordinary differential equation (ODE) 
\[
\frac{dh_t}{dt} = f (t, h_t, \theta) .
\]
After the appearance of the neural ODE, neural SDE is introduced by taking random effect into account, which is described as follows.
Adding noise using the standard Brownian motion $B = \{ B_t : t\in [0,T]\}$ of the form 
\[
dh_t = f(t, h_t, \theta) dt + g(t, h_t, \theta) dB_t .
\]
In neural SDE, an essential part is learned in the drift network $f$, while external noise effect is learned in the diffusion network $g$. 
Our proposal, fSDE-Net, generalizes the noise modeled by a standard Brownian motion to fBm $B^H$:
\[
dh_t = f(t, h_t, \theta) dt + g(t, h_t, \theta) dB^H_t .
\]
Here recall that fBm matches the standard one when $H=1/2$ so that is contains the standard one as a special case.
In particular, fSDE-Net we introduce in this paper is a generalization of the previous SDE-Net.
Moreover, it is known that he noise structure of time series data in many fields such as hydrology, telecommunications, economics, and finance can be described by the fBm~\cite{biagini2008stochastic} and thus fSDE plays a crucial role to describe the dynamics of such types of series with complexity.
From this point of view, our model is useful for generating time series with long-term memory property that appear in various fields of the real world.

\subsection{Numerical Scheme}
To implement neural fSDE, we need to numerically solve fSDE \eqref{eq:fSDE}. The classical Euler scheme works for this purpose when Hurst index is larger than $H>1/2$. We recall here a result which assures the convergence of discretized solutions. 
In the sequel, we denote by $\lfloor a \rfloor $ the maximal integer which does not exceed a real number $a$.
Let $X^n = \{ X^n_{i/n} : i = 0, \ldots, \lfloor nT \rfloor \}$ be discrete series defined by $X^n_0 = X_0$ and the recursive relation  
\[
\begin{aligned}
 X^n_{(i+1)/n } 
 = X^n_{i/n} &+ b(X^n_{i/n}) \frac{1}{n} \\
&+ \sigma(X^n_{i/n}) (B^H_{(i+1)/n} - B^H_{i/n}) .
\end{aligned}
\]
Then we extend $X^n$ to a continuous-time process on $[0,T]$ as a step function, which is still denoted by the same notation: $X^n_t = X^n_{\lfloor nt \rfloor /n}$ for each $t \in [0,T]$. 
In this setting, we have the following convergence result of the discretized solution $X^n$ to the continuous one. (See \cite{neuenkirch2007exact} for the proof.)

\begin{proposition}
\label{prop:numerical}
Assume $H>1/2$ and the drift and volatility functions satisfy $b \in C^2_b(\mathbb{R}) $ and $\sigma \in C^3_b(\mathbb{R})$. Let $X = \{ X_t : t \in [0,T] \}$ be a unique solution of \eqref{eq:fSDE} and let $X^n = \{ X^n_t : t \in [0, T] \} $ be defined as above. Then we have that 
\[
n^{2H-1} \sup_{0 \le t \le T} | X^n_t - X_t | 
\]
converges almost surely to a constant random variable as $n$ tends to infinity. 
\end{proposition}
In other words, the discretized solution converges to the solution of \eqref{eq:fSDE} with precision order $O(n^{1-2H})$, which vanishes if $H>1/2$.

\subsection{Backpropagation}
Towards realization of an adjoint method as implemented for the ODE-Net, we mention here differentiation of fSDE solution with respect to parameters. We consider a process with parameters taking values in some space $\Theta$ whose time evolution is described by the following fractional SDE with Hurst index $H>1/2$. 
\begin{equation}
\label{eq:fsde_param}
dX_t = b(\theta , X_t) dt + \sigma (\theta, X_t ) dB^H_t  
\end{equation}
for each $\theta \in \Theta$. Here we assume the parameter space $\Theta $ is an open set in Euclidean space for simplicity. Let $\mathcal{L}: C([0,T]) \to \mathbb{R}$ be a functional acting on each continuous process. For fSDE-Net, the functional $\mathcal{L}$ is the total loss whose gradient with respect to neural net parameters $\theta $ needs to be calculated. In this situation, if the loss functional satisfies some good properties, then according to the chain rule we have that $\partial_\theta \mathcal{L}(X_\cdot) (\theta) = D\mathcal{L} (X_\cdot) \circ \partial_\theta X_\cdot (\theta )$ for each $\theta \in \Theta$ where $D\mathcal{L}$ denotes the Fr\'{e}chet derivative of the functional $\mathcal{L}$ which maps each element in $C([0,T])$ onto the space of linear functionals on $C([0,T])$. Since in the pathwise approach the equation \ref{eq:fsde_param} is interpreted as an ODE, the following result can easily be proved.  

\begin{theorem}
\label{thm:backprop}
Let $X_t = X_t^\theta$ be a unique solution of \eqref{eq:fsde_param} for each $\theta \in \Theta$. Assume $b, \sigma : \Theta \times \mathbb{R} \to \mathbb{R}$ are differentiable in $\theta \in \Theta$ and $x \in \mathbb{R}$, and assume $b_\theta (\theta, \cdot) , b_x(\theta, \cdot), \sigma_\theta(\theta, \cdot), \sigma_x(\theta, \cdot) \in C^2_b (\mathbb{R})$ for each $\theta$. Then for each $\theta \in \Theta$, the process $Y_t = \partial_\theta X_t $ satisfies the following equation. 
\begin{equation}
\label{eq:backprop}
\begin{aligned}
dY_t & = (b_\theta (\theta, X_t) + b_x (\theta, X_t) Y_t ) dt \\
& \quad + (\sigma_\theta (\theta, X_t) + \sigma_x (\theta, X_t) Y_t ) dB^H_t .
\end{aligned}
\end{equation}
\end{theorem}
\begin{proof}
First write the fSDE \eqref{eq:fsde_param} in the integral form. 
Then by the assumption, differentiation under the integral sign is applicable so that \eqref{eq:backprop} is derived. 
Also, well-posedness of \eqref{eq:backprop} is assured by the assumption. 
\end{proof}

\begin{algorithm}[t]
\caption{Optimization of fSDE-Net generator}\label{alg:fsdenet}
\begin{algorithmic}
\Require $\{ X_0,\ldots, X_T \}$ - a realized path of a stochastic process until time horizon $T$, sample size $M$, learning rate $\eta$, number of optimization steps $k$
\Ensure $(\theta_b, \theta_\sigma)$ - the optimized parameter for fSDE-Net generator with some Hurst index $H$
\While{not converged}
\For{$k$ steps}
    \State Let $\{ \hat{X}^{(i)}_0, \ldots \hat{X}^{(i)}_T \}_{i=1}^M $ be the realization of generated paths of size $M$ and let $p_\theta(t, \cdot)$ be the pdf of log-difference process, which is estimated from generated values $\{ \hat{r}^{(i)}_t \}_{i=1}^M$ where $\hat{r}^{(i)}_t = \log (\hat{X}^{(i)}_{t+1} / \hat{X}^{(i)}_t)$. 
    \State Compute the gradient of loss function
    \[
    \mathcal{L}(\theta) = - \frac{1}{T} \sum_{t=0}^T \log p_\theta (t, r_t).
    \]
    \State Descent parameters: $
    \theta \gets \theta - \eta \nabla_{\theta} \mathcal{L} (\theta)$. 
\EndFor
\EndWhile
\end{algorithmic}
\end{algorithm}

\section{Generative Modeling of Time Series Using fSDE-Net}\label{sec_generative}
\subsection{Modeling with fBm}
One approach to generating synthetic time-series data is a latent model, which particularly decodes input data for dimensionality reduction. 
Instead, we focus on one-dimensional path and directly model observed data here.
Since modeling based on fBm is closely related to the financial market, we give a short description on the model and terms here.  
In mathematical finance, time series are usually modeled by using SDE driven by the standard Brownian motion. 
On the other hand, to describe a more general situation, several models using fBm have been proposed. 
For instance, stock price process $S = \{ S_t : t \ge 0 \}$ is modeled by Fractional Black-Scholes model (see \cite{willinger1999stock} for motivation) and Long-memory stochastic volatility model~\cite{comte1998long}.

\begin{itemize}
\item 
Fractional Black-Scholes model (see \cite{willinger1999stock} for motivation) with $H \neq 1/2$
\[
\begin{cases}
\begin{aligned}
& dr_t = r dt \\
& dS_t = r_t S_t dt + \sigma_t S_t dB^H_t 
\end{aligned}
\end{cases}
\]

\item
Long-memory stochastic volatility model (see \cite{comte1998long} for more detail)
\[
\begin{cases}
\begin{aligned}
& dS_t = r S_t dt + \sigma_t S_t dW_t \\
& d\sigma_t = \alpha \sigma_t dt + \beta  dB^H_t 
\end{aligned}
\end{cases}
\]
where $H>1/2$ and $W = \{ W_t : t \ge 0 \}$ is the standard Wiener process. 
\end{itemize}
It is known that when the time evolution of a stock price process is governed by the above models, arbitrage opportunities exist so that fBm enables us to give a model for inefficient market. 
On the other hand, it is also proved that the arbitrage opportunities disappear if there is a transaction cost (for example, minimum amount of time or payment between transactions, see \cite{cheridito2003arbitrage} for more precise description and its proof). 
This makes modeling financial time series with fBm reasonable, which simultaneously enables us to take long-range dependency or regularity of sample paths into account. 
In the sequel, to deal with more general situation, we assume the time evolution of some stochastic process $X = \{X_t\}_{t \ge 0}$, log-price of a stock $X_t = \log S_t$ for instance, is described by the following time-homogeneous fSDE 
\begin{equation}
\label{eq:fSDE_stock}
dX_t = b(\theta, X_t) dt + \sigma(\theta, X_t) dB^H_t 
\end{equation}
where drift $b$ and diffusion $\sigma$ are determined from neural networks (NN).  Hereafter we optimize $b$ and $\sigma$ with respect NN-parameters $\theta$.

\subsection{Network Architecture}
To estimate drift and volatility functions of \eqref{eq:fSDE_stock} by neural networks, we use $L$-layer multi-layer perceptron (MLP) as network architecture\footnote{We use merely MLP for simplicity and checking whether the solver works in a simple setting. We note that we can expect the improvement of fitting by optimizing network architecture.}:
\[
h^{\ell}_i = \sum_{j=1}^{N_{\ell-1}} w_{ij}^{\ell-1} x^{\ell-1}_j 
+ b^{\ell-1}_i , \quad 
x^{\ell}_i = \varphi(h^\ell_i) 
\]
for each $\ell = 1,\ldots, L$ and each $i = 1, \ldots, N_\ell$. In this case $x^0$ and $x^L$ be input and output data and NN function is such that $x^L = f_\theta (x^0)$ where $\theta = \{ w_{ij}^\ell, b^\ell_i \}_{i, j, \ell}$. In the sequel, let $\Theta$ be the space of NN parameters on which $\theta$ takes values. Then existence and uniqueness of fractional SDE, and precision assurance of numerical solutions are established as follows. 

\begin{theorem}[Informal]
The fSDE-Net generator \eqref{eq:fSDE_stock} with $H>1/2$ whose network architecture is given by MLP with $\mathrm{tanh}$ activation function has a unique solution, which can be numerically solved by the explicit Euler scheme.  
\end{theorem}
\begin{proof}
Note that the activation function $\mathrm{tanh}$ is smooth and its derivatives of any order are bounded. Thus, noting composition of functions inherits regularity and boundedness, we see that neural-net functions $f_\theta$ satisfy the assumptions in Proposition \ref{prop:fSDEsolution} and \ref{prop:numerical} for any $\theta \in \Theta$, which completes the proof.
\end{proof}

Here we note that the assumption on the network architecture is crucial to establish the well-posedness of differential equations. 
Indeed, if drift or diffusion function which drives the SDE grows faster than liner functions, then the solution may explode before in finite time and it becomes impossible to solve it numerically.
On the other hand, when the Lipschitz continuity of driving functions is not satisfied, uniqueness of solution may not hold even for the (deterministic) ODE case. 
In short, both the boundedness and the regularity assumptions on activation functions are crucial to assure the solvablility of neural differential equations. 
From this point of view, tanh is a good choice of activation function while one should be careful when using other functions with low regularity such as ReLU or hard tanh.\footnote{Importance of linear growth condition is pointed out also for discrete settings from the perspective of dynamical isometry. In \cite{pennington2018emergence}, it is proved that derivative of activation function should be bounded by one to circumvent the vanishing gradient problem for infinitely deep neural networks.}

\begin{table*}
  \caption{Numerical results for each performance metric and generative model for fractional OU process~(fOU) and other types of data from the real world. Bold indicates the best performance.}
  \label{tab:result}
  \centering
  \small
\begin{tabular}{ccccccc}
Data & Method & Hurst Index & Marginal Distribution & ACF & Weighted ACF & $R^2$ Score\\ 
\hline \hline
\multirow{4}{*}{fOU(0.7)} 
& Original & 0.7 (True value) & - & - & - & - \\
& RNN & 0.465 $\pm$ 0.154 & 0.194 $\pm$ 0.025 & \bf{1.220} & \bf{0.668} & -1.053 $\pm$ 0.177 \\  
& SDE & 0.608 $\pm$ 0.207 & 0.743 $\pm$ 0.189 & 4.043 & 3.261 & -2.023 $\pm$ 3.042 \\ 
& fSDE & \bf{0.618 $\pm$ 0.156} & \bf{0.328 $\pm$ 0.201} & 2.038 & 1.505 & \bf{-0.955 $\pm$ 0.509} \\ 
\hline
\multirow{4}{*}{fOU(0.8)} 
& Original & 0.8 (True value) & - & - & - & - \\
& RNN  & 0.455 $\pm$ 0.124 & 1.128 $\pm$ 0.027 &\bf{1.260} & \bf{0.668} & -14.072 $\pm$ 1.389 \\
& SDE  & 0.445 $\pm$ 0.137 & 0.946 $\pm$ 0.076 &1.342 &0.750 & -8.087 $\pm$ 1.301 \\
& fSDE & \bf{0.681 $\pm$ 0.151} & \bf{0.441 $\pm$ 0.116} & 1.673 & 1.173 & \bf{-0.374 $\pm$ 0.231} \\
\hline
\multirow{4}{*}{fOU(0.9)} 
& Original & 0.9 (True value) & - & - & - & - \\
& RNN & 0.490 $\pm$ 0.105 &  1.070 $\pm$ 0.026 & \bf{1.582} & \bf{0.985}      &  -17.470 $\pm$ 2.217 \\
& SDE &  0.495 $\pm$ 0.119 &  1.619 $\pm$ 0.065 & 1.709 & 1.123      &  -204.401 $\pm$ 56.750 \\
& fSDE & \bf{0.866 $\pm$ 0.120} & \bf{0.233 $\pm$ 0.143} & 3.487 & 1.805      &  \bf{-1.934 $\pm$ 1.855} \\ 
\hline
\multirow{4}{*}{SPX} 
& Original & 0.591 & - & - & - & - \\
& RNN &  0.479 $\pm$ 0.094 &  0.473 $\pm$ 0.012 & \bf{2.715} & \bf{1.404} & -3.185 $\pm$ 0.401 \\
& SDE & 0.513 $\pm$ 0.115 &  \bf{0.232 $\pm$ 0.067} & 3.052 & 1.720 & \bf{-2.844 $\pm$ 1.282} \\
& fSDE & \bf{0.529 $\pm$ 0.121} &  0.383 $\pm$ 0.062 & 2.785 & 1.469 & -4.036 $\pm$ 1.132 \\
\hline
\multirow{4}{*}{NileMin} 
& Original & 0.973 & - & - & - & - \\
& RNN       & 0.454 $\pm$ 0.117  & 0.848 $\pm$ 0.037 & \bf{1.476} & \bf{0.943}  & -8.798 $\pm$ 1.038 \\
& SDE       & 0.465 $\pm$ 0.139  & \bf{0.225 $\pm$ 0.057} &1.524        &1.021  & -2.041 $\pm$ 0.720 \\
& fSDE & \bf{0.697 $\pm$ 0.179} & 0.549 $\pm$ 0.208       &2.170        &1.556  & \bf{-1.264 $\pm$ 0.892} \\
\hline
\multirow{4}{*}{ethernetTraffic} 
& Original & 0.750 & - & - & - & - \\
& RNN & 0.467 $\pm$ 0.105 & 1.546 $\pm$ 0.008 & \bf{1.647} & \bf{1.125} & -43.778 $\pm$ 1.923 \\
& SDE & 0.590 $\pm$ 0.126  & \bf{1.075} $\pm$ 0.029 & 2.756 & 1.691 & -1.184 $\pm$ 0.168 \\
& fSDE & \bf{0.776 $\pm$ 0.155} & 1.174 $\pm$ 0.067 & 2.787 & 1.630 & \bf{-0.771 $\pm$ 0.427} \\
\hline
\multirow{4}{*}{NBSdiff1kg} 
& Original & 0.707 & - & - & - & - \\
& RNN & 0.463 $\pm$ 0.165  & \bf{0.471 $\pm$ 0.040} & \bf{1.105} & \bf{0.475} & -0.803 $\pm$ 0.273 \\
& SDE & 0.651 $\pm$ 0.200  & 0.532 $\pm$ 0.061 &1.149 & 0.545 & -1.049 $\pm$ 0.406 \\
& fSDE & \bf{0.669} $\pm$ 0.206  & 0.562 $\pm$ 0.107 &1.176 & 0.620 & \bf{-0.538 $\pm$ 0.504} \\
\hline
\multirow{4}{*}{NhemiTemp} 
& Original & 1.066 & - & - & - & - \\
& RNN & 0.479 $\pm$ 0.121 & 0.966 $\pm$ 0.024 & \bf{2.068} & \bf{1.344} & -16.310 $\pm$ 1.118 \\
& SDE & 0.593 $\pm$ 0.166  & 1.148 $\pm$ 0.287 & 4.402 & 2.980 & -70.181 $\pm$ 79.698 \\
& fSDE & \bf{0.742 $\pm$ 0.175} & \bf{0.462 $\pm$ 0.271} & 4.475 & 3.647 & \bf{-1.977 $\pm$ 1.049} \\
\hline
\end{tabular}
\end{table*}

\subsection{Algorithm}
In the sequel, let $\{X_0, \ldots, X_T\}$ be a realized path of a stochastic process, such as log price of a financial product, whose time evolution is assumed to follow some fSDE. 
For numerical reasons, the original data is normalized before learning in such a way that the log-difference process $\{ r_0, \ldots, r_T \}$ where $r_t = \log(X_{t+1}/X_t)$ has zero mean and unit variance. 
After this preprocessing, one can generate synthetic paths $\{ \hat{X}_0,\ldots, \hat{X}_T\}$ by solving fSDE \eqref{eq:fSDE_stock} whose drift and diffusion functions are given by MLP with parameter $\theta\in \Theta$ and then projecting the continuous solution on discrete observation points of the input data. 
In this way, from an input sample path, the fSDE-Net can generate synthetic paths with the same time stamps as the original path that is in general irregularly-sampled. 
We produce $M$ sample paths and denote an $i$-th realization by $\{ \hat{X}^{(i)}_0, \ldots, \hat{X}^{(i)}_T \}$ for each $i = 1, \ldots, M$. 
In the next step, we optimize NN parameters applying the maximal likelihood estimation here, making it possible to produce more realistic paths. 
For each time stamp $t \in \{ 0, 1, \ldots, T\}$, let $p_t (\theta, \cdot)$ be the probability density function of log-difference series $r_t $ on $\mathbb{R}$, which is calculated from generated values $\{ \hat{r}^{(i)}_t \}_{i=1, \ldots, M}$ where $\hat{r}_t = \log(\hat{X}_{t+1} / \hat{X}_t)$ for each $t$. 
Then we maximize the log-likelihood function 
\[
L(\theta) = \frac{1}{T} \sum_{t=0}^T \log p_\theta (t, r_t)
\]
or equivalently minimize a loss $\mathcal{L} (\theta)= - L (\theta)$ with respect to NN parameters $\theta \in \Theta$. 
Here the distribution of log-return is assumed be normal and each likelihood at time $t$ is approximated by using probability density function of the normal distribution whose mean and variance is calculated from generated values of log-difference process $\{ \hat{r}^{(1)}_t, \ldots, \hat{r}^{(M)}_t \}$. 
The generation procedure of the fSDE-Net is given in Algorithm \ref{alg:fsdenet}. 
Here we note that to calculate the gradient of the loss with respect to parameter $\theta$, Theorem \ref{thm:backprop} is helpful for an efficient calculation and for saving memory.

\subsection{Example of Generated Paths}
By using a fractional SDE solver, we can generate artificially time series given a sample path as input.
Figure \ref{fig:fOU0.9_sample_paths} displays a sample path of fractional Ornstein-Uhlenbeck process (see Section \ref{sec_exp} for its definition), compared to generated paths with recurrent neural network~(RNN), SDE-Net and fSDE-Net.
Looking the generated paths, as a natural consequence, one can see that the fSDE-Net inherits regularity of the original path, while the paths generated by other methods has the same regularity as the standard Brownian motion. 
Moreover, since RNN learns time-homogeneous dynamics, note that it is not appropriate for more irregularly-sampled series. 
In the next section, we quantitatively evaluate how our method reproduces characteristics of the original paths taken from various areas, with comparison to RNN and SDE-Net.

\begin{figure}[t]
  \begin{minipage}[b]{0.49\linewidth}
    \centering
    \includegraphics[keepaspectratio, scale=0.27]{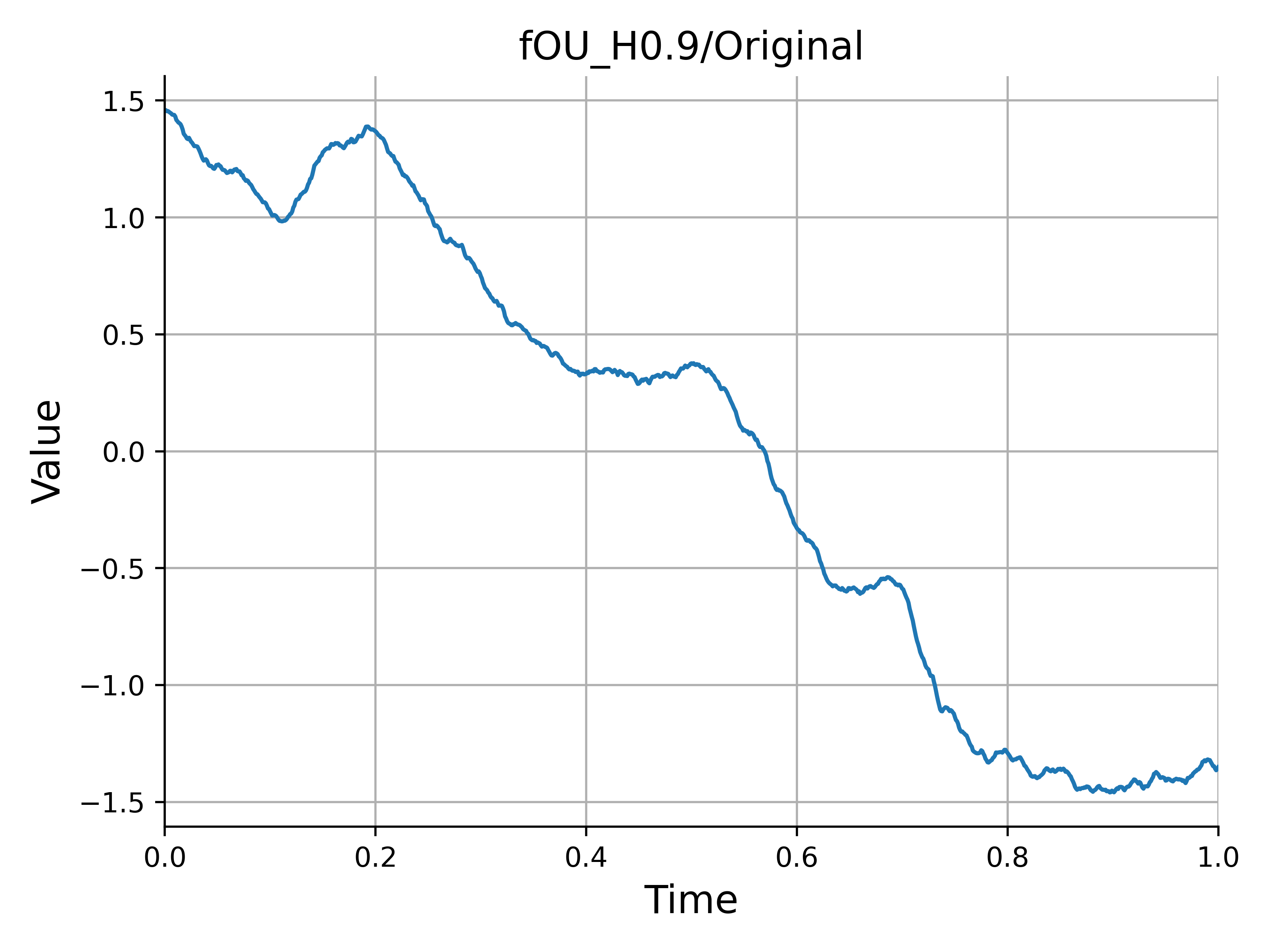}
    \subcaption{Original path}
  \end{minipage}
  \begin{minipage}[b]{0.49\linewidth}
    \centering
    \includegraphics[keepaspectratio, scale=0.27]{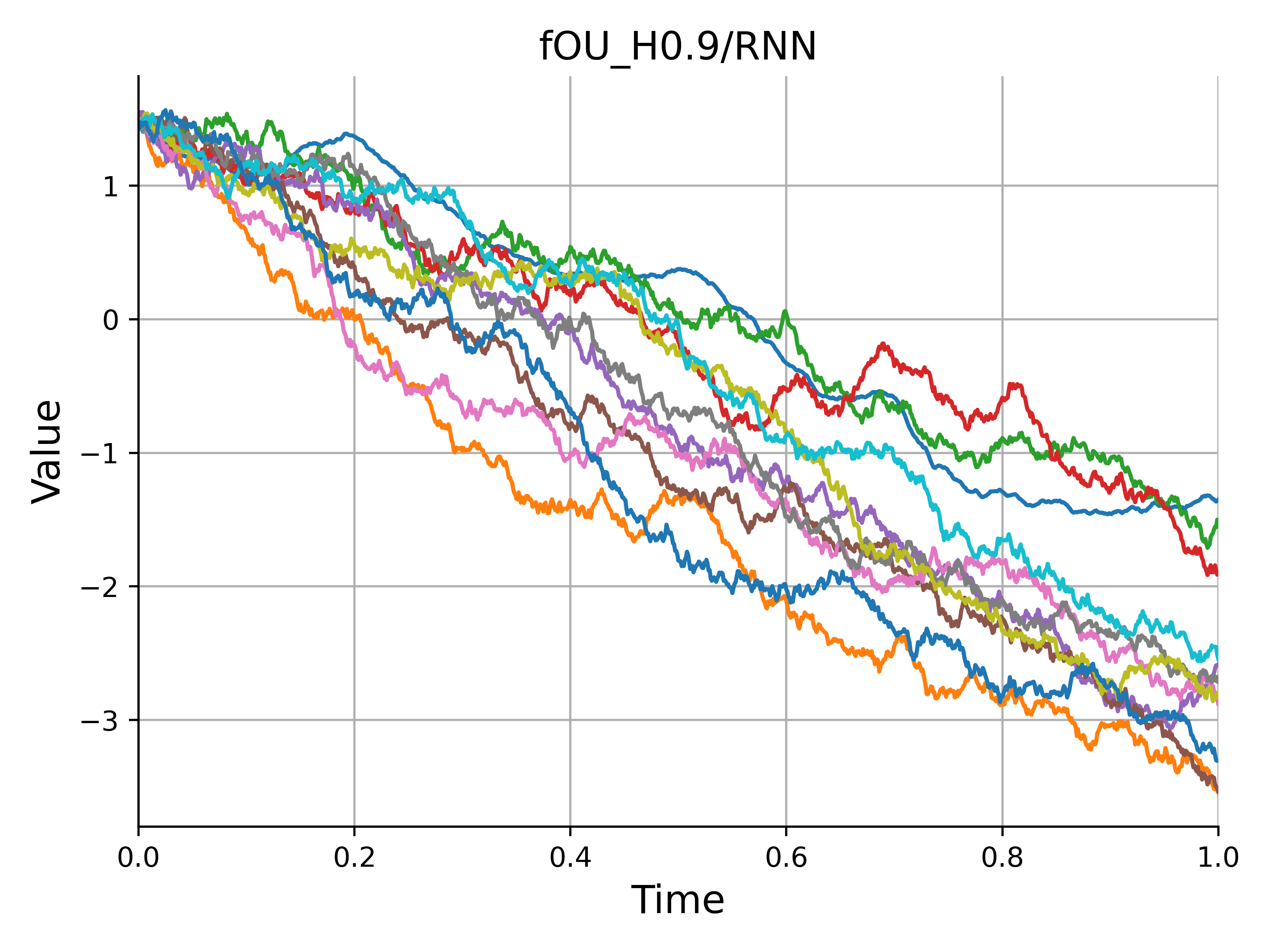}
    \subcaption{Generated paths by RNN}
  \end{minipage}
  \begin{minipage}[b]{0.49\linewidth}
    \centering
    \includegraphics[keepaspectratio, scale=0.27]{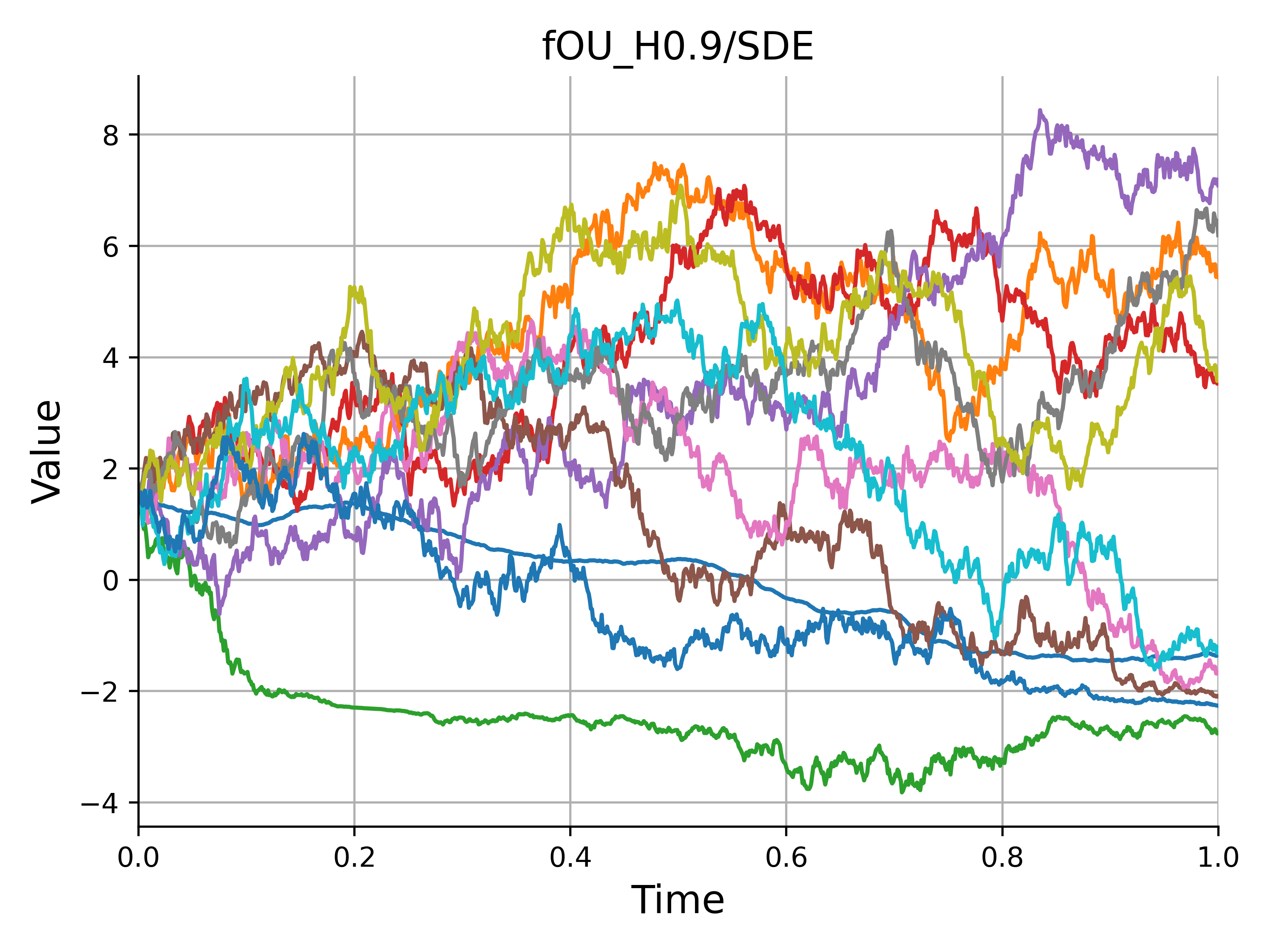}
    \subcaption{Generated paths by SDE-Net}
  \end{minipage}
  \begin{minipage}[b]{0.49\linewidth}
    \centering
    \includegraphics[keepaspectratio, scale=0.27]{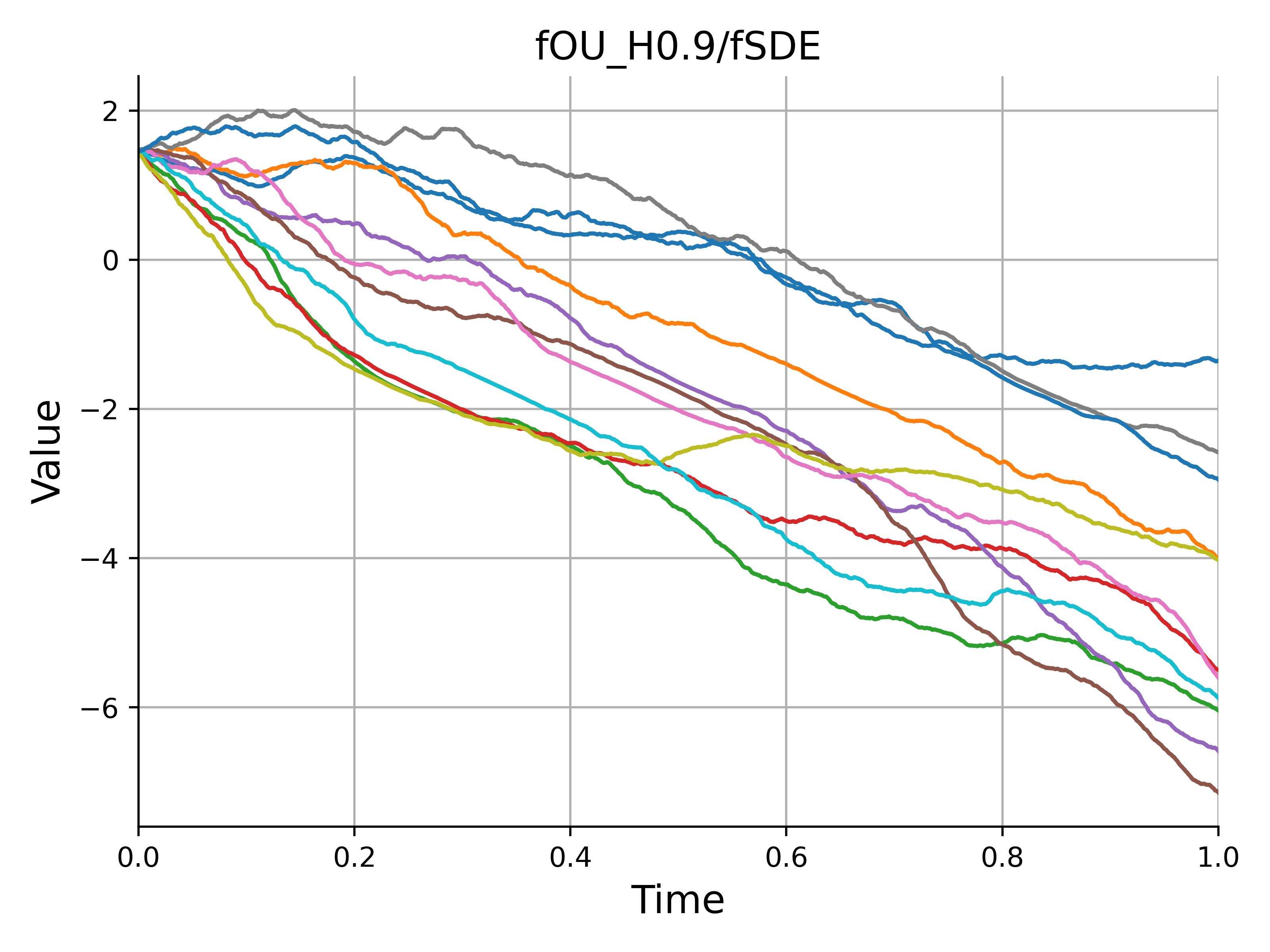}
    \subcaption{Generated paths by fSDE-Net}
  \end{minipage}
\caption{Comparison of the original path of fOU with Hurst index $0.9$ (a) and example of generated paths by RNN (b), SDE-Net (c) and fSDE-Net with $H=0.9$ (d), respectively.}
\label{fig:fOU0.9_sample_paths}
\end{figure}

\section{Experiments}\label{sec_exp}
\subsection{Dataset}
We use the following synthetic and real time-series data as input, from which we will generate new paths with the help of methods using neural networks.  
\subsubsection{Synthetic data} 
As a simple example of series that are compatible with the current situation, we make use of fractional Ornstein-Uhlenbeck (OU) process, with stationarity.
Fractional OU process is a stochastic process whose dynamics is described by the following fSDE.
\[
dX_t = \alpha X_t dt + \beta dB^H_t.
\]
In the sequel, we take $\alpha = -0.05$, $\beta=0.1$ where it is known that OU process is stationary when the parameter $\alpha$ is negative.
Moreover, we set $H=0.7, 0.8$ and $0.9$ to generate a sample path of the process for each value of Hurst index\footnote{We note that generating time series with LRD is not an easy task: (discrete) stochastic processes driven by i.i.d. variables, which are commonly used, do not exhibit LRD of increments. 
Hence we used simply fOU process, which seems adhoc, to ensure LRD of the original path.}. 
We denote them as fOU(0.7), fOU(0.8), and fOU(0.9), respectively.
A realization of each fractional OU process is obtained by dividing the interval $[0,1]$ into 1000 pieces with the common width. 

\subsubsection{Real data}
Also we apply our generation method for time-series data extracted from real world which exhibits long range dependency. 
\begin{itemize}
\item
SPX: Daily closing prices of the S\&P 500 index (SPX) from January 2000 until November 2021.

\item 
NileMin: Yearly minimal water levels of the Nile river for the years 622 to 1281, measured at the Roda gauge near Cairo. 

\item
ethernetTraffic: Ethernet traffic data from a LAN at Bellcore, Morristown. 


\item
NBSdiff1kg: NBS weight measurements - deviation from 1 kg in micrograms.

\item
NhemiTemp: Monthly temperature for the northern hemisphere for the years 1854-1989, from the data base held at the Climate Research Unit of the University of East Anglia, Norwich, England. 
The numbers consist of the temperature (degrees C) difference from the monthly average over the period 1950-1979.
\end{itemize}

The SPX data is obtained from Bloomberg terminal. 
For the SPX series, the estimated Hurst index value by the $R/S$ statistics is $0.591$ similar to previous studies\cite{bayraktar2004estimating}. 
In particular, the SPX data exhibits long-range dependency in this observation interval. The other series from real world are taken from the webpage of CRAN (Comprehensive R Archive Network). 
These can be downloaded in 
https://cran.r-project.org/web/packages/longmemo/index.html
and one can find a precise description of each series.

\subsection{Experimental Setup}
In the sequel, we quantitatively investigate the performance of the fSDE-Net based on some criterion. 
As existing time-series generators for comparison, recurrent neural networks (RNN) and SDE-Net are used for generators instead of fSDE-Net. 
As network architecture, we use 2-layer MLP for SDE- and fSDE-Net with 20 hidden units while for RNN generator we use a vanilla RNN with 40 hidden units.

\begin{figure*}[t]
\begin{minipage}[b]{0.5\linewidth}
    \centering
    \includegraphics[keepaspectratio, scale=0.42]{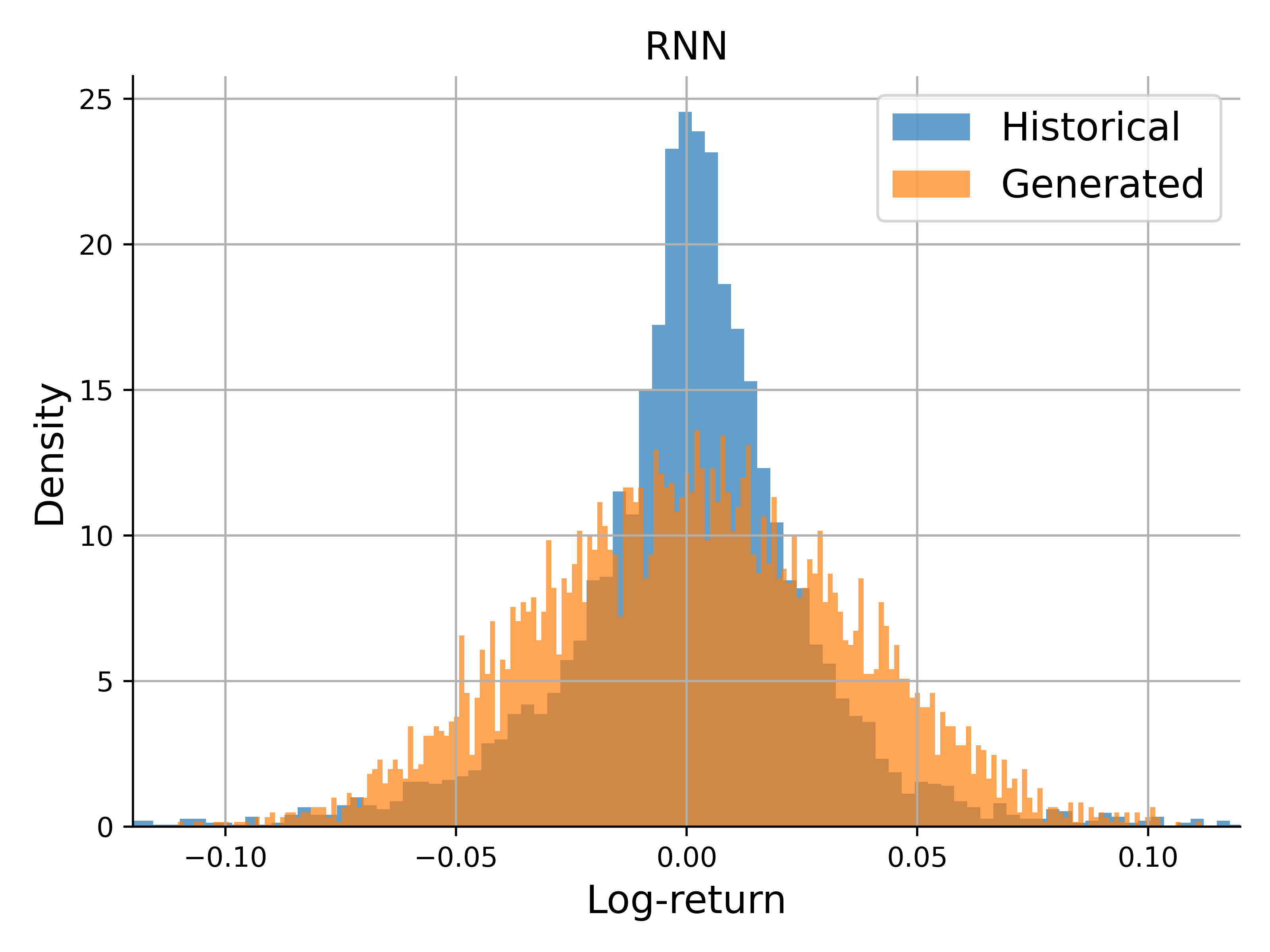}
\end{minipage}
\begin{minipage}[b]{0.5\linewidth}
    \centering
    \includegraphics[keepaspectratio, scale=0.42]{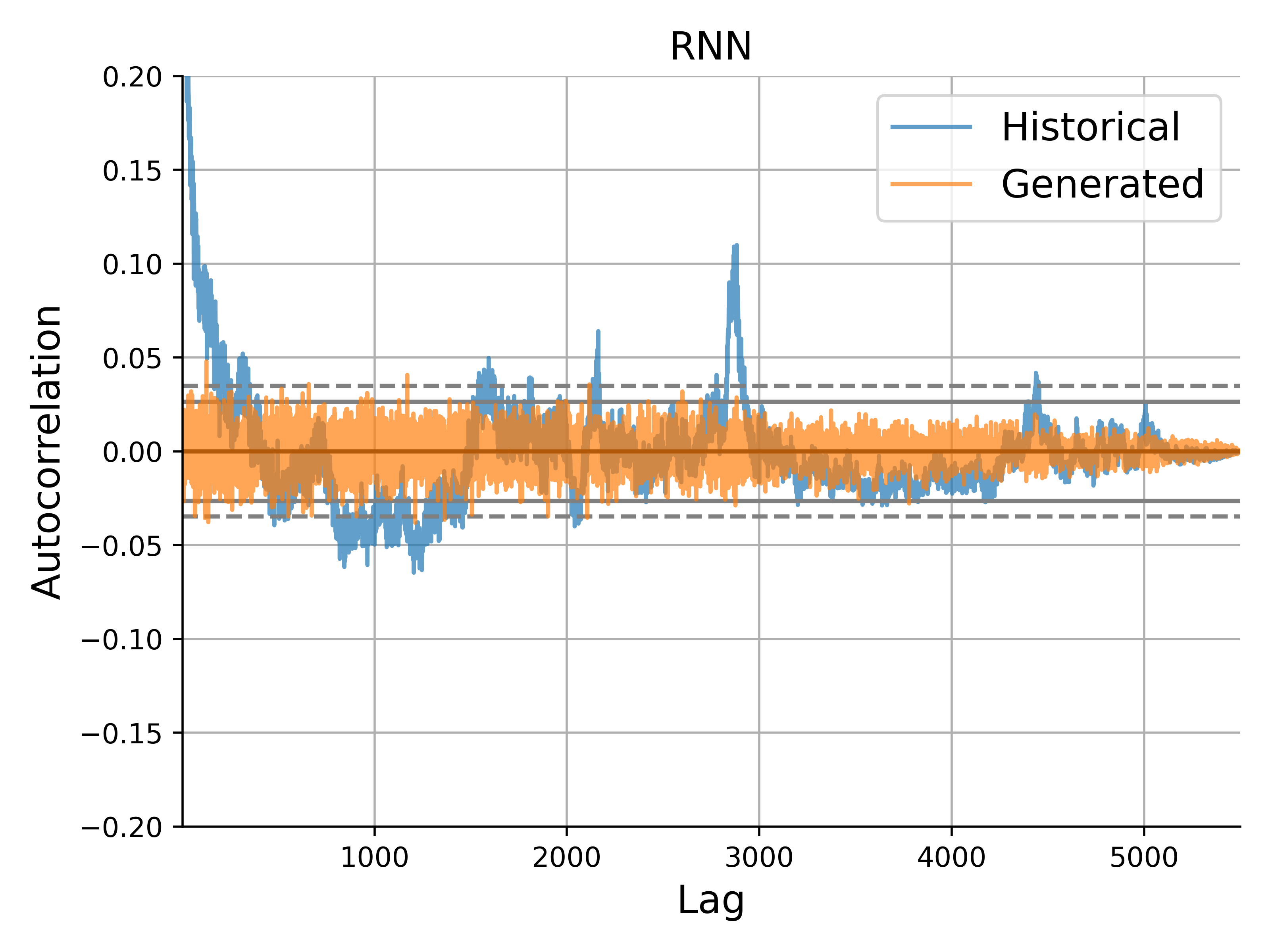}
\end{minipage} \\
\begin{minipage}[b]{0.5\linewidth}
    \centering
    \includegraphics[keepaspectratio, scale=0.42]{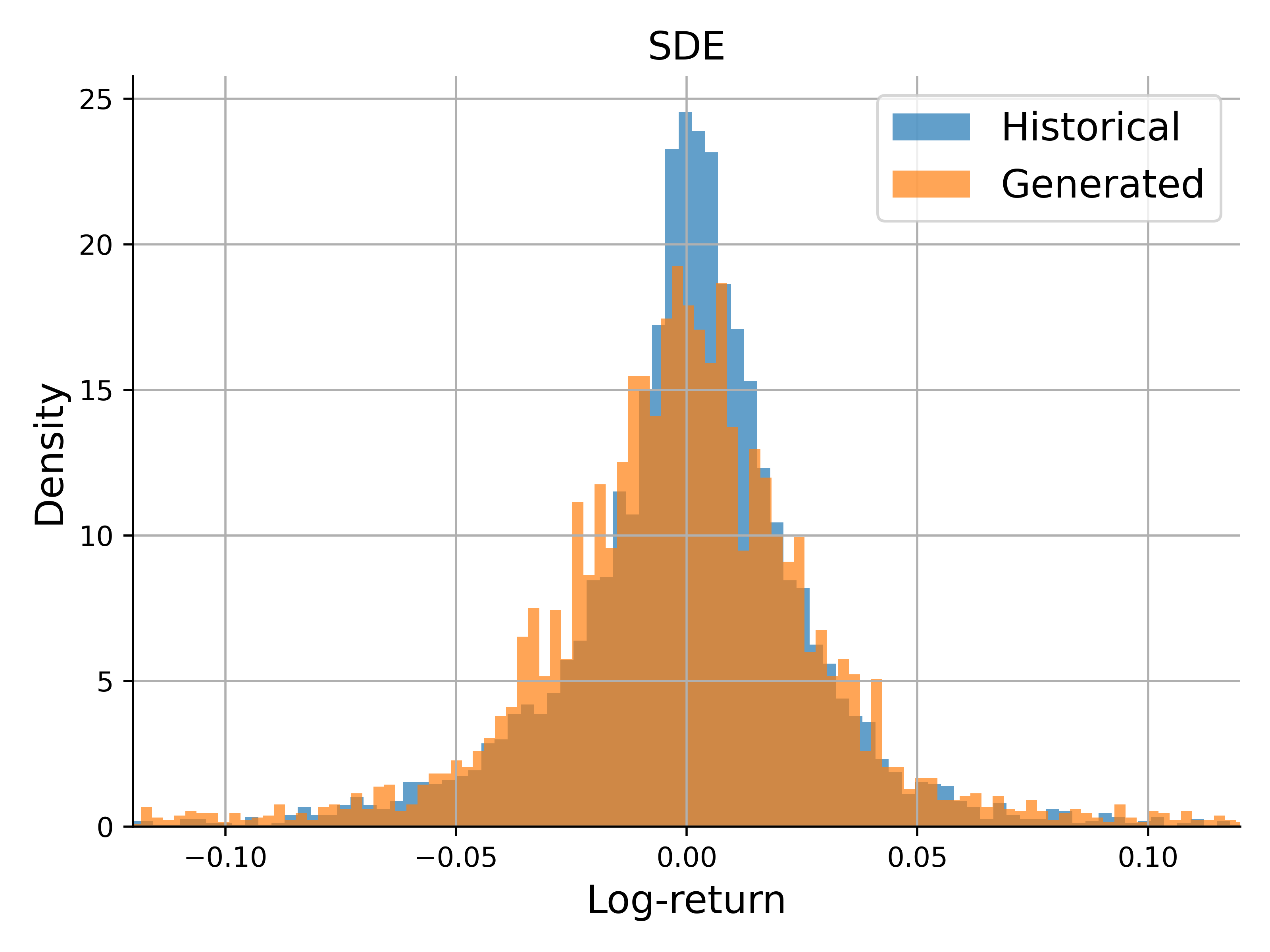}
\end{minipage}
\begin{minipage}[b]{0.5\linewidth}
    \centering
    \includegraphics[keepaspectratio, scale=0.42]{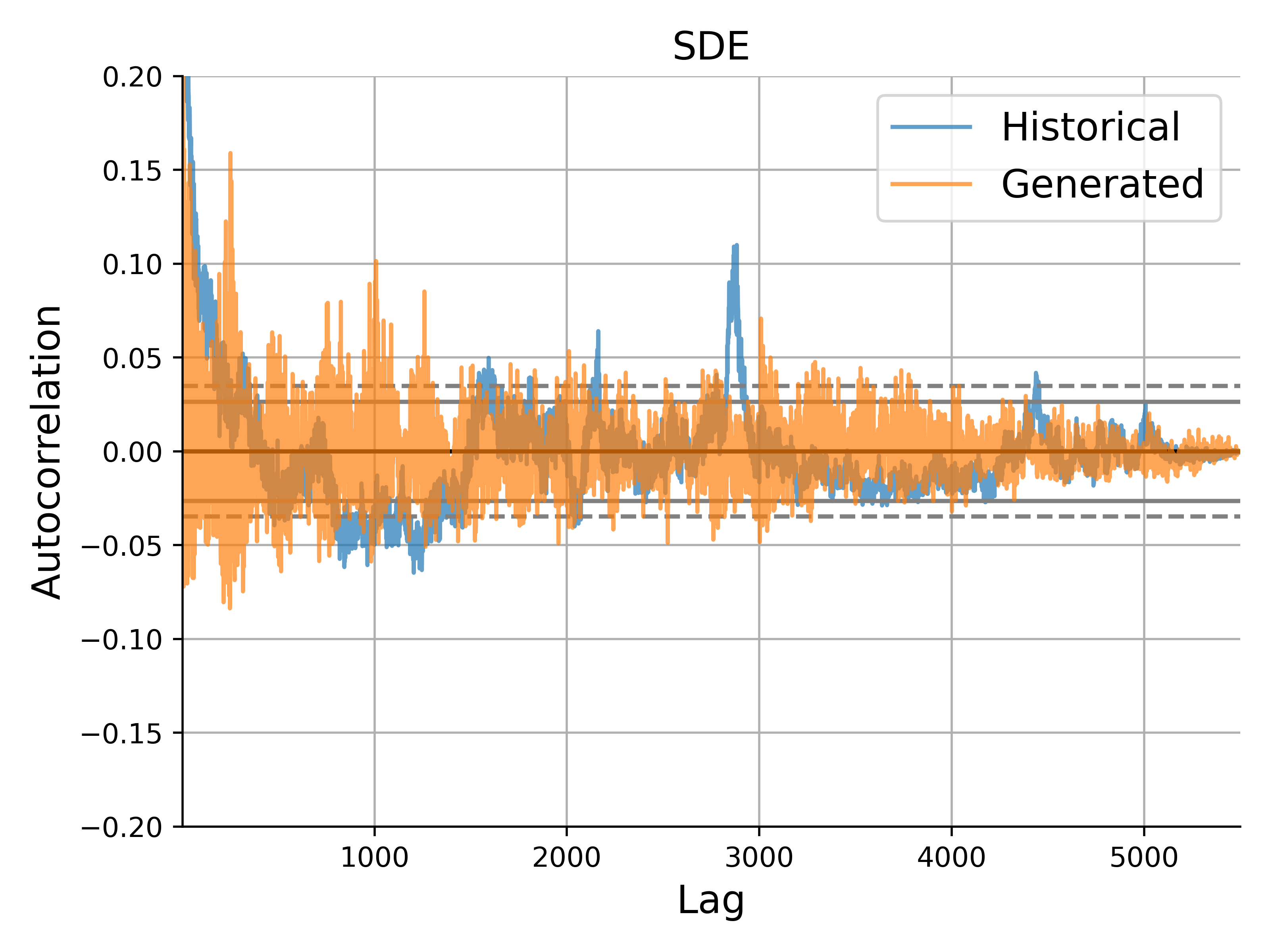}
\end{minipage} \\
\begin{minipage}[b]{0.5\linewidth}
    \centering
    \includegraphics[keepaspectratio, scale=0.42]{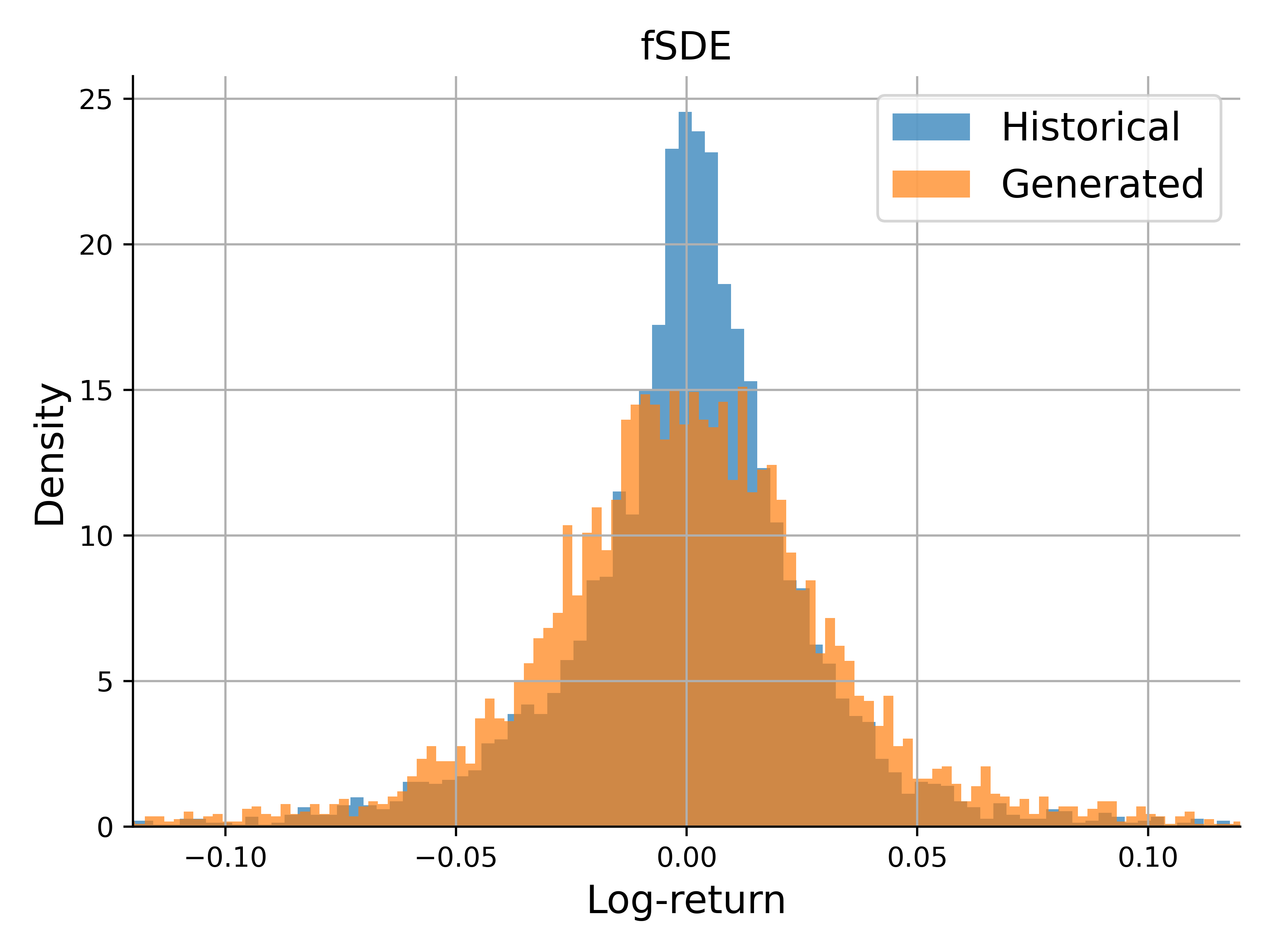}
  \end{minipage}
  \begin{minipage}[b]{0.5\linewidth}
    \centering
    \includegraphics[keepaspectratio, scale=0.42]{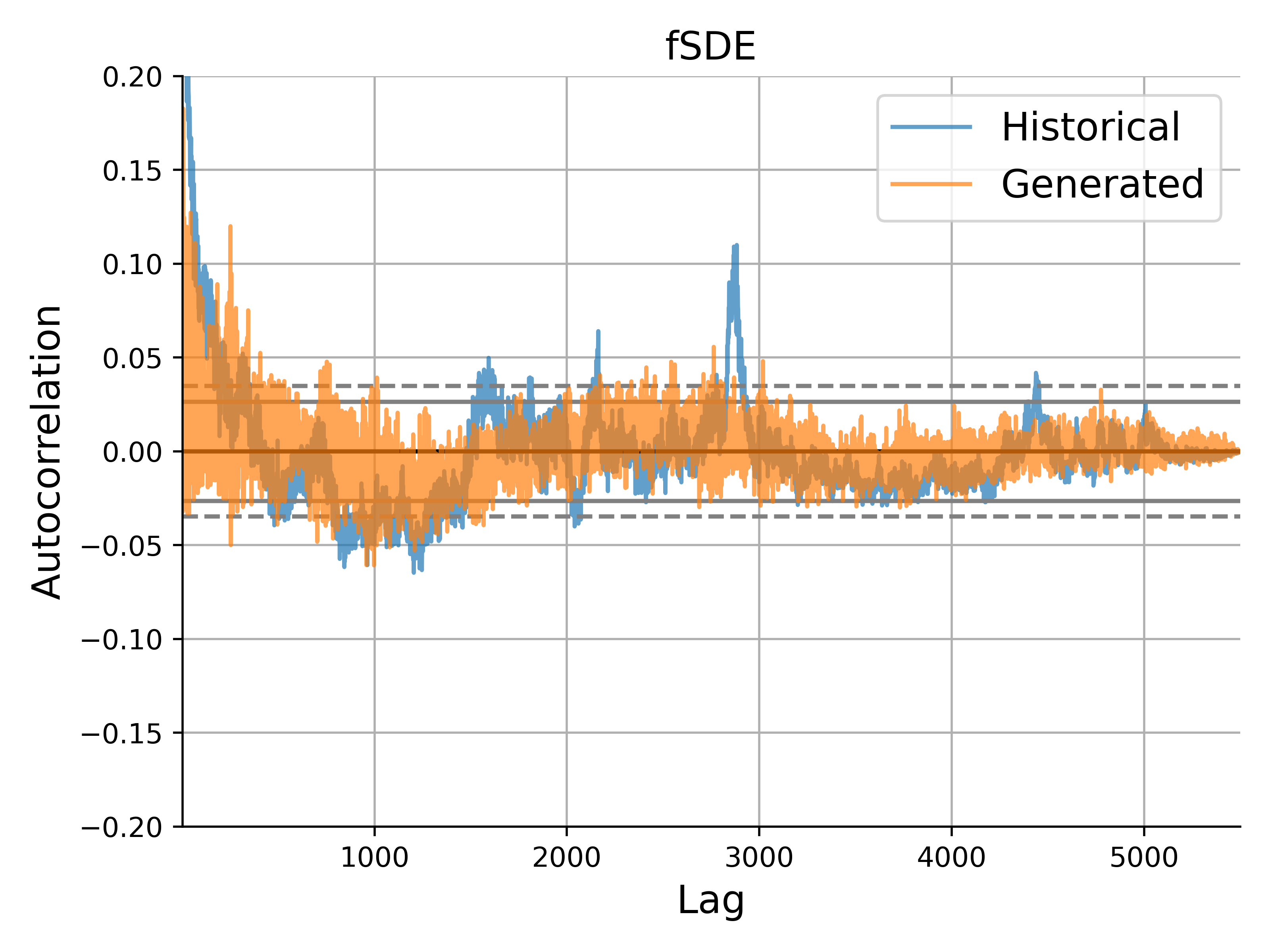}
\end{minipage}
\caption{Comparison in synthetic and original S\&P 500 index as to histogram (left column) and correlogram (right column) of log return process. Synthetic time series are generated after calibration by RNN (upper row), SDE-Net (middle row) and fSDE-Net (lower row).}
\label{fig:hist_acf}
\end{figure*}

\subsection{Performance Metrics}
In addition to check whether generated processes have long-range dependent increments by measuring the Hurst index, we use the following criterion to compare the generated paths with the original one: marginal distribution, dependency, and predictive score.

Since our main purpose is to rebuild the LRD property for synthetic paths, which are generated by the fSDE-Net, the most important performance characteristic is the Hurst index $H$.

\subsubsection{Hurst index}
Estimate Hurst index $H$ by Mandelbrot's method. 
More specifically, recalling $R/S$ statistics satisfy the relation \eqref{eq:RS} for a large time, the Hurst index is estimated via linear regression after eliminating the first $100$ data. 
By comparing the estimated value of Hurst index $H$, we check whether $H > 1/2$ holds true so that the generated time series exhibits the long-range dependency. 

\subsubsection{Marginal distribution}
Compare empirical distributions of historical and generated data by the following manner. 
Let $\rho$ and $\hat{\rho}$ be empirical probability density functions (epdf) of historical and generated paths, respectively. 
Then for a set of bins $\mathcal{B} = \{B_1, \ldots, B_K \} $, we compute  
\[
\frac{1}{2} \sum_{B \in \mathcal{B}} |\rho(B) - \hat{\rho}(B)| ,
\]
as a metric between historical and generated distributions. 
Note that this metric takes values in $[0,1]$ according to the factor $1/2$.  
Here $\rho (B) = |B|^{-1} \int_{B} \rho(x)dx $, which is similarly defined for $\hat{\rho}$, for each $B \in \mathcal{B}$.

\subsubsection{Dependency}
We write $r_{0:T}=(r_0, \ldots, r_T)$. Let $\gamma(\tau) = \mathrm{Corr}(|r_t|, |r_{t+\tau}|)$ denote correlation coefficient of the absolute log-difference series with lag $\tau$ and let $C : \mathbb{R}^T \to [-1, 1]^S: r_{0:T} \mapsto (\gamma(1), \ldots, \gamma(S))$ be the auto-correlation function up to lag $S \le T-1$. 
Then, following \cite{wiese2020quant} we use the ACF score defined by 
\[
\mathrm{ACF} = \bigg\| C(r_{0:T}) - \frac{1}{M} \sum_{1\le i \le M} C(\hat{r}_{0:T}^{(i)}) \bigg\|_2 
\]
where $\| \cdot \|_2$ denotes the Euclidean norm.
In addition, we consider weighted version of the ACF score to test whether auto-correlation decays with heavy tail for large time lags. 
\[
\mathrm{ACF}_w = \bigg\| C(r_{0:T})\circ w 
- \frac{1}{M} \sum_{1\le i\le M} C(\hat{r}^{(i)}_{0:T})\circ w \bigg\|_2 
\]
where $\circ$ denotes the Hadamard product and we use $w = (2i/(S+1))_{i=1,\ldots, S}$ as a weight vector with unit mean. 

\subsubsection{Predictive score}
Finally, to evaluate whether prediction of forthcoming time evolution, we compare the coefficient of determination $R^2$ computed for testing data sets.  
For calculation of the $R^2$ score, we use the first $80\%$ of data for training, while the other data is used for testing.


\subsection{Results}
Table \ref{tab:result} shows results for the above performance metrics after 100 iteration steps where Adam with learning rate 0.04 is used as the optimizer. 
Looking Hurst indices in the table, values for the fSDE-Net generator is higher than those for the other methods for types of data. 
From this perspective, the fSDE-Net outperforms the other generative methods reproducing LRP of the original path. 
Moreover, for the fractional OU process and NhemiTemp, the fSDE-Net generates better marginal distribution. 
On the other hand, in view of values of performance metrics on marginal distribution and ACF, it looks there is no advantage in fSDE-Net comparing to the other methods. (See Figure \ref{fig:hist_acf} that displays histogram and correlogram of a sample path of log-difference sequence generated by RNN, SDE-Net and fSDE-Net.) 
Finally, all $R^2$ scores take negative values, which indicates difficulty of modeling and making prediction on real-world time series. 
However, the $R^2$ score for the fSDE-Net tends take values not so small compared to the other methods. 
This roughly indicates that the fSDE-Net generates synthetic paths that are macroscopically close to the original one.

\section{Conclusion}\label{sec_con}
The main purpose of this paper is generation of time series data, from input series which are often irregularly sampled, using neural differential equation model, which can be viewed as a deep neural network with infinitely many layers.
To broaden the scope of application of the above continuous model, we focused on the case where the noise structure is complicated, varying from usual iid models. 
In particular, we studied heredity of long-term memory property of input time series, which is observed in wide variety of data, when generating data with the help of neural differentials. 
As a generative model which generates artificial paths duplicating the long-term memory property of an input path, we proposed the fSDE-Net, a neural differential equation model using fractional Brownian motion with Hurst index larger than half. 
Moreover, we constructed numerical solver of fractional SDE via the explicit Euler scheme and applied the fSDE-Net to generation of irregularly-sampled time series with long-range dependency. 
According to a theoretical analysis, one should postulate good conditions, regularity and at most linear growth, on driving functions of neural differential equations, which gives a standard way to choose activation function. 
On the other hand, numerical simulation revealed that the fSDE-Net has more ability to generate long-range dependent paths than existing RNN and SDE generators. 
In summary, we revealed that the fSDE-Net we introduced in this paper is a generative model which is applicable for irregularly-sampled time series and inherit the long-range dependency. 

\subsection*{Future Works}
In this paper, since implementation of the fSDE-Net we are proposing is a central role at the first stage, we consider only a simple situation to check the possibility to construct the generator.
Consequently, there is room for further optimizing the network architecture rather than the current MLP with a few layers, for instance, to improve the experimental result. 
As a further development of the research, to investigate the generation of time series with a different characteristic is a possible direction. 
For instance, extension to the case when Hurst index is less than half is important to take roughness of time series into account, albeit there are some difficulties to construct solution theory of fSDE with rough noise and its numerical solver. 
Alternatively, generative latent model is also important version of the present method, since it enables us to reduce dimensions of input data.  

\subsection*{Broader Impact}
An application of generative models specialized in financial time series as the fSDE-Net we proposed in this paper is to evaluate option trading strategies. In this direction, \cite{wiese2019deep} demonstrates GAN-based DNNs generate more realistic time series outperforming the classical ones. Since fBm-driven models are motivated by real data, there is a possibility that our generative method is successfully applied to evaluating or constructing hedge strategies. 
Not restricted to the application to financial market, our generative method broadened a comprehensive way to reproduce complex time series, which can be applied in such a way as model construction, forecasting and anomalous detection, etc. in various areas of the real world.


\bibliography{reference}

\begin{thebibliography}{10}
\providecommand{\url}[1]{#1}
\csname url@samestyle\endcsname
\providecommand{\newblock}{\relax}
\providecommand{\bibinfo}[2]{#2}
\providecommand{\BIBentrySTDinterwordspacing}{\spaceskip=0pt\relax}
\providecommand{\BIBentryALTinterwordstretchfactor}{4}
\providecommand{\BIBentryALTinterwordspacing}{\spaceskip=\fontdimen2\font plus
\BIBentryALTinterwordstretchfactor\fontdimen3\font minus
  \fontdimen4\font\relax}
\providecommand{\BIBforeignlanguage}[2]{{%
\expandafter\ifx\csname l@#1\endcsname\relax
\typeout{** WARNING: IEEEtran.bst: No hyphenation pattern has been}%
\typeout{** loaded for the language `#1'. Using the pattern for}%
\typeout{** the default language instead.}%
\else
\language=\csname l@#1\endcsname
\fi
#2}}
\providecommand{\BIBdecl}{\relax}
\BIBdecl

\bibitem{shumway2000time}
R.~H. Shumway, D.~S. Stoffer, and D.~S. Stoffer, \emph{Time series analysis and
  its applications}.\hskip 1em plus 0.5em minus 0.4em\relax Springer, 2000,
  vol.~3.

\bibitem{goodfellow2014generative}
I.~Goodfellow, J.~Pouget-Abadie, M.~Mirza, B.~Xu, D.~Warde-Farley, S.~Ozair,
  A.~Courville, and Y.~Bengio, ``Generative adversarial nets,'' \emph{Advances
  in neural information processing systems}, vol.~27, 2014.

\bibitem{gui2021review}
J.~Gui, Z.~Sun, Y.~Wen, D.~Tao, and J.~Ye, ``A review on generative adversarial
  networks: Algorithms, theory, and applications,'' \emph{IEEE Transactions on
  Knowledge and Data Engineering}, 2021.

\bibitem{mogren2016c}
O.~Mogren, ``C-rnn-gan: Continuous recurrent neural networks with adversarial
  training,'' \emph{arXiv preprint arXiv:1611.09904}, 2016.

\bibitem{esteban2017real}
C.~Esteban, S.~L. Hyland, and G.~R{\"a}tsch, ``Real-valued (medical) time
  series generation with recurrent conditional gans,'' \emph{arXiv preprint
  arXiv:1706.02633}, 2017.

\bibitem{yoon2019time}
J.~Yoon, D.~Jarrett, and M.~van~der Schaar, ``Time-series generative
  adversarial networks,'' \emph{Advances in Neural Information Processing
  Systems}, vol.~32, pp. 5508--5518, 2019.

\bibitem{wiese2020quant}
M.~Wiese, R.~Knobloch, R.~Korn, and P.~Kretschmer, ``Quant gans: Deep
  generation of financial time series,'' \emph{Quantitative Finance}, vol.~20,
  no.~9, pp. 1419--1440, 2020.

\bibitem{ni2020conditional}
H.~Ni, L.~Szpruch, M.~Wiese, S.~Liao, and B.~Xiao, ``Conditional
  sig-wasserstein gans for time series generation,'' \emph{arXiv preprint
  arXiv:2006.05421}, 2020.

\bibitem{ni2021sig}
H.~Ni, L.~Szpruch, M.~Sabate-Vidales, B.~Xiao, M.~Wiese, and S.~Liao,
  ``Sig-wasserstein gans for time series generation,'' \emph{arXiv preprint
  arXiv:2111.01207}, 2021.

\bibitem{nakagawa2020ric}
K.~Nakagawa, M.~Abe, and J.~Komiyama, ``Ric-nn: a robust transferable deep
  learning framework for cross-sectional investment strategy,'' in \emph{2020
  IEEE 7th International Conference on Data Science and Advanced Analytics
  (DSAA)}.\hskip 1em plus 0.5em minus 0.4em\relax IEEE, 2020, pp. 370--379.

\bibitem{chen2018neural}
R.~T. Chen, Y.~Rubanova, J.~Bettencourt, and D.~Duvenaud, ``Neural ordinary
  differential equations,'' in \emph{Advances in Neural Information Processing
  Systems}, 2018, pp. 6572--6583.

\bibitem{he2016deep}
K.~He, X.~Zhang, S.~Ren, and J.~Sun, ``Deep residual learning for image
  recognition,'' in \emph{Proceedings of the IEEE conference on computer vision
  and pattern recognition}, 2016, pp. 770--778.

\bibitem{rubanova2019latent}
Y.~Rubanova, R.~T. Chen, and D.~Duvenaud, ``Latent odes for irregularly-sampled
  time series,'' \emph{arXiv preprint arXiv:1907.03907}, 2019.

\bibitem{pontryagin1987mathematical}
L.~S. Pontryagin, \emph{Mathematical theory of optimal processes}.\hskip 1em
  plus 0.5em minus 0.4em\relax CRC press, 1987.

\bibitem{karatzas2012brownian}
I.~Karatzas and S.~Shreve, \emph{Brownian motion and stochastic
  calculus}.\hskip 1em plus 0.5em minus 0.4em\relax Springer Science \&
  Business Media, 2012, vol. 113.

\bibitem{revuz2013continuous}
D.~Revuz and M.~Yor, \emph{Continuous martingales and Brownian motion}.\hskip
  1em plus 0.5em minus 0.4em\relax Springer Science \& Business Media, 2013,
  vol. 293.

\bibitem{mandelbrot1968fractional}
B.~B. Mandelbrot and J.~W. Van~Ness, ``Fractional brownian motions, fractional
  noises and applications,'' \emph{SIAM review}, vol.~10, no.~4, pp. 422--437,
  1968.

\bibitem{biagini2008stochastic}
F.~Biagini, Y.~Hu, B.~{\O}ksendal, and T.~Zhang, \emph{Stochastic calculus for
  fractional Brownian motion and applications}.\hskip 1em plus 0.5em minus
  0.4em\relax Springer Science \& Business Media, 2008.

\bibitem{banna2019fractional}
O.~Banna, Y.~Mishura, K.~Ralchenko, and S.~Shklyar, \emph{Fractional Brownian
  motion: approximations and projections}.\hskip 1em plus 0.5em minus
  0.4em\relax John Wiley \& Sons, 2019.

\bibitem{rostek2013note}
S.~Rostek and R.~Sch{\"o}bel, ``A note on the use of fractional brownian motion
  for financial modeling,'' \emph{Economic Modelling}, vol.~30, pp. 30--35,
  2013.

\bibitem{greene1977long}
M.~T. Greene and B.~D. Fielitz, ``Long-term dependence in common stock
  returns,'' \emph{Journal of Financial Economics}, vol.~4, no.~3, pp.
  339--349, 1977.

\bibitem{tzen2019theoretical}
B.~Tzen and M.~Raginsky, ``Theoretical guarantees for sampling and inference in
  generative models with latent diffusions,'' in \emph{Conference on Learning
  Theory}.\hskip 1em plus 0.5em minus 0.4em\relax PMLR, 2019, pp. 3084--3114.

\bibitem{tzen2019neural}
------, ``Neural stochastic differential equations: Deep latent gaussian models
  in the diffusion limit,'' \emph{arXiv preprint arXiv:1905.09883}, 2019.

\bibitem{rezende2014stochastic}
D.~J. Rezende, S.~Mohamed, and D.~Wierstra, ``Stochastic backpropagation and
  approximate inference in deep generative models,'' in \emph{International
  conference on machine learning}.\hskip 1em plus 0.5em minus 0.4em\relax PMLR,
  2014, pp. 1278--1286.

\bibitem{kong2020sde}
L.~Kong, J.~Sun, and C.~Zhang, ``Sde-net: Equipping deep neural networks with
  uncertainty estimates,'' in \emph{International Conference on Machine
  Learning}.\hskip 1em plus 0.5em minus 0.4em\relax PMLR, 2020, pp. 5405--5415.

\bibitem{kidger2021neural}
P.~Kidger, J.~Foster, X.~Li, H.~Oberhauser, and T.~Lyons, ``Neural sdes as
  infinite-dimensional gans,'' \emph{arXiv preprint arXiv:2102.03657}, 2021.

\bibitem{liu2019neural}
X.~Liu, S.~Si, Q.~Cao, S.~Kumar, and C.-J. Hsieh, ``Neural sde: Stabilizing
  neural ode networks with stochastic noise,'' \emph{arXiv preprint
  arXiv:1906.02355}, 2019.

\bibitem{peluchetti2020infinitely}
S.~Peluchetti and S.~Favaro, ``Infinitely deep neural networks as diffusion
  processes,'' in \emph{International Conference on Artificial Intelligence and
  Statistics}.\hskip 1em plus 0.5em minus 0.4em\relax PMLR, 2020, pp.
  1126--1136.

\bibitem{jia2019neural}
J.~Jia and A.~R. Benson, ``Neural jump stochastic differential equations,''
  \emph{Advances in Neural Information Processing Systems}, vol.~32, pp.
  9847--9858, 2019.

\bibitem{herrera2021neural}
C.~Herrera, F.~Krach, and J.~Teichmann, ``Neural jump ordinary differential
  equations: Consistent continuous-time prediction and filtering,'' in
  \emph{International Conference on Learning Representations}, 2021.

\bibitem{kidger2020neural}
\BIBentryALTinterwordspacing
P.~Kidger, J.~Morrill, J.~Foster, and T.~Lyons, ``Neural controlled
  differential equations for irregular time series,'' in \emph{Advances in
  Neural Information Processing Systems}, vol.~33, 2020, pp. 6696--6707.
  [Online]. Available:
  \url{https://proceedings.neurips.cc/paper/2020/file/4a5876b450b45371f6cfe5047ac8cd45-Paper.pdf}
\BIBentrySTDinterwordspacing

\bibitem{mandelbrot1972statistical}
B.~Mandelbrot, ``Statistical methodology for nonperiodic cycles: from the
  covariance to r/s analysis,'' in \emph{Annals of Economic and Social
  Measurement, Volume 1, Number 3}.\hskip 1em plus 0.5em minus 0.4em\relax
  NBER, 1972, pp. 259--290.

\bibitem{young1936inequality}
L.~C. Young, ``An inequality of the h{\"o}lder type, connected with stieltjes
  integration,'' \emph{Acta Mathematica}, vol.~67, no.~1, p. 251, 1936.

\bibitem{neuenkirch2007exact}
A.~Neuenkirch and I.~Nourdin, ``Exact rate of convergence of some approximation
  schemes associated to sdes driven by a fractional brownian motion,''
  \emph{Journal of Theoretical Probability}, vol.~20, no.~4, pp. 871--899,
  2007.

\bibitem{willinger1999stock}
W.~Willinger, M.~S. Taqqu, and V.~Teverovsky, ``Stock market prices and
  long-range dependence,'' \emph{Finance and stochastics}, vol.~3, no.~1, pp.
  1--13, 1999.

\bibitem{comte1998long}
F.~Comte and E.~Renault, ``Long memory in continuous-time stochastic volatility
  models,'' \emph{Mathematical finance}, vol.~8, no.~4, pp. 291--323, 1998.

\bibitem{cheridito2003arbitrage}
P.~Cheridito, ``Arbitrage in fractional brownian motion models,'' \emph{Finance
  and Stochastics}, vol.~7, no.~4, pp. 533--553, 2003.

\bibitem{pennington2018emergence}
J.~Pennington, S.~Schoenholz, and S.~Ganguli, ``The emergence of spectral
  universality in deep networks,'' in \emph{International Conference on
  Artificial Intelligence and Statistics}.\hskip 1em plus 0.5em minus
  0.4em\relax PMLR, 2018, pp. 1924--1932.

\bibitem{bayraktar2004estimating}
E.~Bayraktar, H.~V. Poor, and K.~R. Sircar, ``Estimating the fractal dimension
  of the s\&p 500 index using wavelet analysis,'' \emph{International Journal
  of Theoretical and Applied Finance}, vol.~7, no.~05, pp. 615--643, 2004.

\bibitem{wiese2019deep}
M.~Wiese, L.~Bai, B.~Wood, and H.~Buehler, ``Deep hedging: learning to simulate
  equity option markets,'' \emph{arXiv preprint arXiv:1911.01700}, 2019.

\end{thebibliography}

\end{document}